\documentclass{article}

    \PassOptionsToPackage{numbers, compress}{natbib}

\usepackage[preprint]{neurips_2026}

\usepackage[utf8]{inputenc} %
\usepackage[T1]{fontenc}    %
\usepackage{hyperref}       %
\usepackage{url}            %
\usepackage{booktabs}       %
\usepackage{amsfonts}       %
\usepackage{nicefrac}       %
\usepackage{microtype}      %
\usepackage{xcolor}         %

\title{Representation Learning for Sample-Efficient CATE Estimation by Leveraging Multiple Outcomes}

\author{
Maitreyi Swaroop\textsuperscript{1}\thanks{Corresponding author: \texttt{mswaroop@andrew.cmu.edu}},
  Shikha Bhat\textsuperscript{1},
  Samantha Rodriguez\textsuperscript{2},
  Tamar Krishnamurti\textsuperscript{2},
  Bryan Wilder\textsuperscript{1}
  \\[4pt]
  \textsuperscript{1}Machine Learning Department, Carnegie Mellon University \\
  \textsuperscript{2}Department of General Internal Medicine, University of Pittsburgh
}
\usepackage{microtype}
\usepackage{graphicx}
\usepackage{subcaption}
\usepackage{booktabs} 
\usepackage{hyperref}
\usepackage{tikz}
\usetikzlibrary{shapes,decorations,arrows,calc,arrows.meta,fit,positioning}
\tikzset{
    -Latex,auto,node distance =1 cm and 1 cm,semithick,
    state/.style ={ellipse, draw, minimum width = 0.7 cm},
    point/.style = {circle, draw, inner sep=0.04cm,fill,node contents={}},
    bidirected/.style={Latex-Latex,dashed},
    el/.style = {inner sep=2pt, align=left, sloped}
}

\usepackage{amsmath}
\usepackage{amssymb}
\usepackage{mathtools}
\usepackage{amsthm}
\usepackage{bm}

\usepackage[capitalize,noabbrev]{cleveref}

\theoremstyle{plain}
\newtheorem{theorem}{Theorem}[section]

\newtheorem{lemma}[theorem]{Lemma}

\theoremstyle{definition}

\newtheorem{assumption}[theorem]{Assumption}
\crefname{assumption}{assumption}{assumptions}
\theoremstyle{remark}
\newtheorem{remark}[theorem]{Remark}

\usepackage{bm}
\newcommand\E{\mathbb{E}}

\newcommand\Rd{\mathbb{R}^{d}}

\newcommand\Rm{\mathbb{R}^{m}}

\DeclareMathOperator*{\argmax}{arg\,max}
\DeclareMathOperator*{\argmin}{arg\,min}

\usepackage{xcolor}
\usepackage{threeparttable}
\definecolor{SafeGreen}{RGB}{0,100,0}

\begin{document}

\maketitle

\begin{abstract}
Estimating conditional average treatment effects (CATE) enables efficient targeting of interventions, but many applications have limited experimental samples, making it difficult to estimate heterogeneous effects from high-dimensional covariates. In such events, policy-makers and medical 
practitioners often succumb to the curse of dimensionality or apply off-the-shelf dimension reduction methods that may not preserve treatment heterogeneity. Yet these domains often come with large historical datasets measuring a wide range of outcomes -- a source of supervision that is rarely exploited in practice. Following causal representation learning, we hypothesize that such domains with high-dimensional covariates have lower-dimensional underlying dynamics. We can thus leverage the diverse outcomes measured in historical data to learn a lower-dimensional representation of the covariates. Theoretically, we prove that when the auxiliary outcomes satisfy a set of surrogacy conditions and the representation retains relevant covariate information, the original CATE is identified when the high-dimensional covariates are replaced by the learned representation. Combined with existing dimension-dependent rates for CATE estimation, the result implies greater sample-efficiency on the same experimental sample. Additionally, we characterize the bias-variance tradeoff when the assumptions do not hold perfectly, and show that the representation-based estimator can still achieve lower error when the reduction in estimator variance outweighs the bias due to compression. Empirically, we evaluate the method on synthetic data and semi-synthetic medical data.
\end{abstract}

\section{Introduction}\label{sec:intro}
 Treatment effects often vary substantially across individuals, and identifying those who benefit most enables efficient allocation of limited resources \cite{athey2025machine}. Estimating the conditional average treatment effect (CATE) is therefore central to designing well-targeted interventions. 
 However, CATE estimation faces a fundamental challenge: randomized experiments are often small-scale, while the covariates needed to capture meaningful heterogeneity are high-dimensional. 
 This curse of dimensionality often leads researchers to abandon data-driven heterogeneity discovery in favor of manually specifying subgroups of interest \cite{athey2016recursive, chernozhukov2018generic,varadhan2013framework}. 
In machine learning more broadly, pretraining has emerged as a solution to sample-limited tasks \cite{kolesnikov2020big,brown2020language,radford2021learning}. The insight to learn representations from large, related datasets that transfer effectively to specific use cases  has transformed many areas of machine learning \cite{bommasani2022opportunitiesrisksfoundationmodels}. Yet, application to treatment effect estimation remains limited \cite{nilforoshan2023zero,liu2024cure,li2024large}, and often confined to specialized settings where covariates are images or text \cite{veitch2020adapting,feder2022causal,jiang2023estimating}. 

Outside of treatment effect estimation, causal representation learning (CRL) \cite{scholkopf2021toward} studies how high-dimensional observations arise from lower-dimensional causal variables, often with the goal of recovering those variables up to an equivalence class. We build upon this perspective, but use representation learning to aid statistical estimation of causal quantities, instead of identifying the ground-truth causal variables. Recent work on actionable prediction \cite{liu2024actionability} supports the view that effective intervention requires measuring latent states that determine treatment heterogeneity (rather than simply optimizing single-outcome prediction). For example, effects of a mental health intervention may be driven by social determinants (access to family or community support), biological factors (genetic predisposition or chronic health conditions), and other stressors (food security or employment status). The role of these latent factors in driving health outcomes is well-studied in the healthcare and social sciences literature \citep{bollen2002latent,phelan2010social,adler2010health}. While these states are often not directly observed, EHR data provides rich observational samples where these latent factors are (noisily) reflected across many measured covariates and outcome indicators.
 
We propose that the diverse outcomes often measured in administrative or historical data can help uncover this latent structure and enable more accurate estimation of heterogeneous treatment effects. Theoretically, we use the surrogacy framework of \citet{athey2019surrogate} to prove that when the auxiliary outcomes satisfy a set of surrogacy conditions, the learned representation will suffice to capture heterogeneous treatment effects on the target outcome.
To realize this, we learn a bottleneck representation $\phi(X)$ on large-scale historical data which captures information between the covariates $X$ and $k$ auxiliary outcomes $(Y^{(1)},\!\ldots\!, Y^{(k)})$. These outcomes serve as indicators of underlying dynamics that predict heterogeneity in the target outcome $Y^{*}$.
We then use data from a randomized experiment of the intervention of interest to learn treatment effects on a target outcome $Y^{*}$ as a function of the lower-dimensional $\phi(X)$ instead of the original $X$. Combining our identification result with the CATE estimation error rates of \citet{kennedy2023towards}, our method provides gains in sample-efficiency as CATE estimation rates degrade with covariate dimension. 

We emphasize that we do not require surrogacy conditions to be exactly satisfied for our method to be useful (just as pretraining yields sample efficiency gains for many tasks even if the pretrained representations are not perfectly adapted to the target domain). When the assumptions do not hold perfectly, the estimand {coarsens} to the average treatment effect among units with covariates sharing the same representation. The total error in CATE estimation admits the familiar bias-variance decomposition of squared bias (due to coarsening), and variance (due to the finite experimental sample). In practice, the total error is often minimized by accepting some nonzero bias in this tradeoff. Thus even when the conditions only hold approximately, we find that our lower dimensional representation achieves lower total error than the raw covariates.

Empirically, we demonstrate our method on synthetic and semi-synthetic medical data, finding substantial improvements in sample efficiency for CATE estimation.   

The paper is organized as follows: \Cref{sec:related-work} discusses related work from causal inference literature. \Cref{sec:formulation} presents the formal problem setup and introduces notation. \Cref{sec:main-result} provides the main identification result of the CATE estimator using the learned representation. \Cref{sec:experiments} provides empirical evaluation of our method with both synthetic and real datasets. \Cref{sec:conclusion} concludes.

\section{Related Work}\label{sec:related-work}
\paragraph{Heterogeneous treatment effect estimation} 
CATE estimation has advanced through meta-learner frameworks: the S-learner and T-learner estimate potential outcomes, the X-learner \cite{kunzel2019metalearners} is particularly effective for unbalanced treatment groups, and the R-learner \cite{nie2021quasi} builds on Double Machine Learning \cite{chernozhukov2018double} for robustness to nuisance estimation errors. Parallel to these frameworks, causal forests \cite{wager2018estimation}, extending causal trees \cite{athey2016recursive}, are standard for non-parametric CATE estimation. 
However, in high-dimensional settings with limited experimental samples, convergence rates degrade \cite{curth2024using}, necessitating methods that learn lower-dimensional representations that extract relevant structure from the original covariates. 
Existing remedies such as sufficient dimension reduction \citep{ghosh2021sufficient}, or integrating out nuisance covariates over a pre-selected subset \cite{fan2022estimation}, require manual selection or linearity assumptions, motivating the need for methods that automatically learn lower-dimensional representations.
\paragraph{Causal representation learning}
Work at the intersection of representation learning and causal inference falls into two categories. The first learns representations for treatment effect estimation from observational data: balanced representations to reduce the covariate shift between treatment groups \cite{johansson2016learning,shalit2017estimating}, latent variable models for unobserved confounding and selection bias \citep{louizos2017causal, hassanpour2019learning,zhang2020treatment}, shared representations between propensity and outcome models \citep{shi2019adapting}, and energy-based representations \cite{zhang2022identifiable}.
However, these methods operate within a single dataset and do not leverage the multi-outcome structure available in historical data. 
The second category (the more common interpretation of causal representation learning) aims to recover latent causal variables from high-dimensional observations \citep{scholkopf2021toward}. This includes identifiable nonlinear ICA \citep{hyvarinen2016unsupervised, khemakhem2020variational}, and work on identifying latent causal structure under interventions \citep{brehmer2022weakly}. These methods establish conditions under which latent variables can be recovered, but do not address treatment effect estimation.
\paragraph{Leveraging auxiliary information}
Our approach leverages the availability of multiple auxiliary outcomes to uncover latent structure. While this setting naturally encompasses the \textit{surrogacy framework} where intermediate outcomes mediate the effect of treatment on a long-term target \cite{Prentice1989SurrogateEI, athey2019surrogate}, our motivation differs. Prior surrogacy literature addresses two main problems, (1) \textit{identification} - where long-term outcomes are unobserved and surrogates enable estimation of otherwise unidentifiable effects \cite{athey2019surrogate, imbens2025long,cai2024long}, including recent extensions to heterogeneous treatment effects \cite{cai2025long}, and (2) \textit{variance reduction}, where surrogates improve efficiency when the primary outcome is partially observed \cite{kallus2025role}. However, none of this work uses surrogacy structure to learn representations that transfer to new experiments for sample-efficient CATE estimation. Our method applies even when the primary outcome is fully observed, and where identification is not a concern but sample efficiency is. \\
Related directions include \textit{data fusion}, which combines experimental data with larger observational datasets to address selection bias or increase statistical power \cite{yang2025data}, and \textit{transfer learning} for causal inference, which focuses on generalizing treatment effects across populations with different covariate distributions \cite{pearl2022external, kunzel2018transfer, bica2022transfer}. Our setting differs from both as we do not assume that the treatment of interest appears in the historical data at all, nor that the experimental and historical samples come from different populations. Instead, we assume rich historical data from the same population, collected before the intervention was available.

\paragraph{Our contribution}
The methods discussed above address concerns about the validity of the observational or experimental sample, such as unmeasured confounding, selection bias, or covariate shift between populations. %
Our goal differs in that we target the setting where the experimental design is internally valid but statistically underpowered due to limited samples and high-dimensional covariates, with experimental and historical covariates drawn from the same population. 
We use multi-outcome historical data to learn a representation that helps transfer this structural knowledge to the experimental sample, thereby reducing the effective dimensionality of the inference problem.
We thus enable accurate CATE estimation on small experimental samples where standard methods would suffer from the curse of dimensionality. 

\section{Problem Formulation}\label{sec:formulation}
We consider treatment effect estimation in a setting where experiments are expensive or limited in sample size. To overcome this limitation, we leverage a larger historical dataset that captures the underlying outcome dynamics, to learn a lower dimensional representation of the covariates.

\paragraph{Datasets} We have access to two datasets (or \textit{populations} $P\in \{H, E\}$):
\begin{itemize}
    \item \textit{Historical dataset ($P=H$):} A large sample (size $N_H$) collected prior to the experiment. It comprises covariates $X\in \Rd$, $k$ auxiliary outcomes $S=(Y^{(1)},\dots,Y^{(k)})$ (\textit{surrogates}, following the surrogacy literature), and a target outcome $Y^{*}$. 
    \item \textit{Experimental dataset ($P=E$):} A relatively smaller sample (size $N_E$) collected from a randomized controlled trial (RCT). It comprises covariates $X$, assigned treatments $T$, and auxiliary outcomes $S$.
\end{itemize}

We do not assume that units were assigned a treatment in the historical dataset. For example, administrative data may predate an intervention, as is common in policy and health settings. We also do not require the primary outcome $Y^*$ to be immediately observable in the experimental sample - for instance, when $Y^*$ is a long-term outcome and treatment effects must be estimated before $Y^*$ can be measured. 
Under Assumptions \ref{assump:surrogacy}-\ref{assump:comparability}, observing $Y^*$ in the experimental sample is not required for identification. However, when $Y^*$ is observed, our reliance on surrogacy assumptions weakens. The conditions are sufficient to ensure that $\phi(X)$ captures the full CATE, but are not necessary to estimate a valid causal effect on $Y^{*}$. 

\paragraph{Estimand}
Our goal is to learn a lower-dimensional representation $\phi: \Rd\to \Rm$ of the covariates $X$ where $m\ll d$, using the observations $(X_{i}, S_{i}, Y^{*}_{i}, P_i=H)_{i=1}^{N_H}$. Let $Y(t)$ denote the potential outcome in treatment $t\in \{0,1\}$ (for both the target and the auxiliary outcomes). We will then use $\phi$ to estimate the CATE on the target outcome in the experimental dataset $(X_{i}, T_{i}, S_{i}(T_{i}), P_{i}=E)_{i=1}^{N_E}$. Our target estimand is 
\begin{equation}
\tau(x) = \mathbb{E}[Y^{*}(1)-Y^{*}(0)|X=x, P=E]
\end{equation}
which our method will learn by estimating
\begin{equation}
\mathbb{E}[Y^{*}(1)-Y^{*}(0)|\phi(X), P=E].
\end{equation}

\paragraph{Assumptions}
We organize our assumptions from the causal inference literature into two categories: (1) standard identification assumptions, which we will assume hold throughout and (2) assumptions from the surrogacy framework of Athey et al. \cite{athey2019surrogate}, which we will use to analyze when pretrained representations are lossless, but for which we will also study estimation error under violations. 
\paragraph{\textit{Standard causal inference assumptions.}} (Standard in treatment effect estimation literature).
\begin{assumption}[Stable Unit Treatment Value Assumption (SUTVA)]\label{assump:sutva}
    We assume that the SUTVA \cite{rubin1980sutva} holds, ensuring well-defined potential outcomes. 
\end{assumption}
\begin{assumption}[Unconfoundedness/Ignorability]\label{assump:unconfounded}
\[(X, Y^{(j)}(1),Y^{(j)}(0)) \perp T | P=E \quad\text{for all outcomes } j \text{ (auxiliary and target).}\] This assumption is satisfied by design since $P=E$ is a randomized controlled trial. We adopt unconditional ignorability as our primary assumption. The extension to stratified randomization, where ignorability holds conditional on a known subset $W \subset X$, is discussed in Appendix~\ref{app:stratified}.
\end{assumption}  
\begin{assumption}[Overlap] \label{assump:overlap} 
\begin{itemize}
    \item[(i)] $0<\Pr[T=1 \mid X, P=E] < 1$ for all $X$.
    \item[(ii)] The support of the experimental covariates is contained within the historical support, i.e., $\text{supp}(P_E(X)) \subseteq \text{supp}(P_H(X))$.
\end{itemize}
\end{assumption}
\paragraph{\textit{Surrogacy assumptions.}} (From the surrogacy framework of \citet{athey2019surrogate}).
\begin{assumption}[Surrogacy]\label{assump:surrogacy}
$T \perp Y^{*} \mid S, X, P=E$.
\end{assumption}
\begin{assumption}[Comparability]\label{assump:comparability} $Y^{*} \perp P \mid S, X$

\textit{Surrogacy} requires that $S$ fully mediates the treatment effect on $Y^{*}$, while \textit{comparability} ensures the conditional distribution of $Y^{*}$ given $(S, X)$ is consistent across populations.
\end{assumption}

\section{Methodology}\label{sec:main-result}
\begin{figure}[t]
    \centering
    \includegraphics[width=\linewidth]{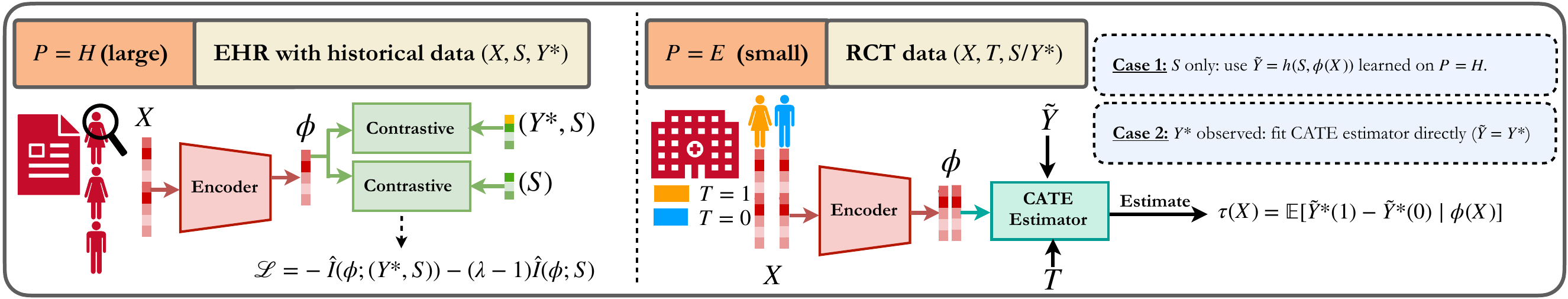}
    \caption{\textbf{Method overview.} On historical data we train $\phi$ via contrastive lower bounds on $I(\phi;(Y^*,S))$ and $I(\phi;S)$ (Equation~\ref{eq:objective_function}). On the RCT we feed $(\phi(X), T, \tilde Y)$ to a CATE meta-learner, with $\tilde Y = h(S, \phi(X))$ when $Y^*$ is unobserved (Case 1) or $\tilde Y = Y^*$ when it is observed (Case 2).}
    \label{fig:method-overview}
\end{figure}
In this section, we establish the conditions under which a learned representation $\phi(X)$ suffices to identify heterogeneous treatment effects, and derive a training objective that encourages the representation to retain outcome-relevant information while discarding irrelevant variation in $X$.

Before specifying what $\phi$ must satisfy, we observe that ignorability is preserved under any compression of $X$ (see Lemma~\ref{lem:ignorability} in Appendix~\ref{app:ignorability} for the formal statement and proof). This distinguishes our setting from observational representation learning \citep{johansson2016learning,shalit2017estimating}, where $\phi$ must preserve confounders to maintain ignorability, whereas randomization removes this requirement in our setting. We now turn to the question of what additional structure $\phi$ must capture for sufficiency.

\subsection{Representation requirements}\label{sec:rep_requirements}
The following assumptions specify the conditions a representation must satisfy to preserve all the information relevant for CATE estimation. We use this to motivate our training objective (\Cref{sec:objective_function}).

\begin{assumption}[Sufficiency]\label{assump:sufficiency}
\begin{itemize}
    \item[(i)] $Y^{*}\perp X|\phi(X),S,P=H$
    \item[(ii)] $S\perp X | \phi(X), T, P=E$
\end{itemize}
Condition (i) requires that, given representation $\phi(X)$ and surrogates $S$, the full covariates $X$ provide no additional information about the target outcome $Y^{*}$. Condition (ii) requires that, given $\phi(X)$ and treatment $T$, the covariates provide no additional information about the surrogates. Together, these ensure $\phi(X)$ captures all covariate information relevant for predicting treatment effect heterogeneity.
\end{assumption}

\begin{remark}[Testability and proxy objective]\label{rem:testability}
     Condition (i) is testable on historical data; Condition (ii) must hold for all treatment levels. Since historical data contains only controls ($T=0$), we cannot enforce this condition for treated units during representation learning; we instead optimize a proxy objective (\Cref{sec:objective_function}). The full condition can be verified post-hoc on experimental data, though statistical power is limited by the smaller experimental sample size.
\end{remark}

\subsection{Identification result}
Our main theoretical contribution shows that under Assumptions ~\ref{assump:sutva}-~\ref{assump:sufficiency}, the representation $\phi(X)$ learned from historical data suffices to identify the $X$-conditional heterogeneous treatment effects.%
\begin{theorem}[Identification of CATE under the learned representation]\label{thm:cate_id}
    Let $h(s,z) = \mathbb{E}[Y^{*}|S=s, \phi(x)=z, P=H]$. Under assumptions ~\ref{assump:sutva}-~\ref{assump:sufficiency},
    \begin{equation}\label{eq:our_cate_estimator}
        \tau(x) =\mathbb{E}\left[h(S, \phi(x)) \mid \phi(x), T=1, P=E\right] - \mathbb{E}\left[h(S, \phi(x)) \mid \phi(x), T=0, P=E\right]
    \end{equation}
\end{theorem}
\begin{remark}
The function $h$ is learned from historical data where we observe $(X,S,Y^{*})$ with all units untreated. Under surrogacy (\ref{assump:surrogacy}), the structural equation for $Y^{*}$ depends on $(X,S)$ but not on $T$, allowing the same $h$ to apply to both treatment arms.
Comparability (\ref{assump:comparability}) ensures this relationship is the same across populations.
\end{remark}
\begin{remark}
Crucially, the outer expectation in \Cref{thm:cate_id} marginalizes over the distribution of $S$ conditional on $\phi(x)$. This implies that to estimate $\tau(x)$ for a specific unit, we only require their covariates $X$, not their post-treatment surrogate outcomes.
\end{remark}

\paragraph{Proof intuition.} 
The large historical sample identifies how $(X,S)$ predict the target outcome $Y^*$, while the randomized experiment sample identifies how the treatment changes the distribution of $S$ (and, if available, $Y^*$) conditional on $X$. Comparability allows the historical outcome relationship between $(X,S)$ and $Y^{*}$ to be used in the experimental sample, and surrogacy allows for treatment-induced changes in $S$ to be translated into changes in $Y^*$. The two sufficiency conditions justify replacing $X$ with $(\phi(X),S)$ in both steps: condition (i) ensures that $\phi(X)$ captures all covariate information needed to predict $Y^{*}$ from surrogates, so the outcome model $h$ loses nothing by conditioning on $\phi(X)$ instead of $X$; condition (ii) ensures that $\phi(X)$ captures all covariate information needed to predict surrogates given treatment, so estimating surrogate shifts conditional on $\phi(X)$ recovers the same quantity as conditioning on $X$. Notably, no treatment variation is needed to learn the representation and it is trained entirely on historical controls. 
Full proof in Appendix~\ref{app:cate_id_proof}.

\paragraph{Implications for treatment effect estimation}
The identification result suggests a natural two-stage estimation procedure (see Figure~\ref{fig:method-overview}). In the first stage, we use historical data to (i) learn the representation $\phi$ by optimizing Equation~\ref{eq:objective_function}, and (ii) fit the outcome model $h$ mapping $\phi(X), S$ to $Y^{*}$. In the second stage, we estimate treatment effects on the experimental sample using standard CATE meta-learners, but with a key modification that pseudo outcomes are regressed on $\phi(X)$ rather than the raw covariates $X$. This is where the sample efficiency gain is realized -- Theorem~\ref{thm:cate_id} shows that, under sufficiency, the original CATE can be estimated using $\phi(X)\in\mathbb{R}^m$ rather than $X\in\mathbb{R}^d$, with $m\ll d$. Under the conditions of Corollary~1 of \citet{kennedy2023towards}, the DR-Learner for a $\gamma$-smooth CATE then has pointwise error $O_p\!\left(N_E^{-\gamma/(2\gamma+m)}\right),$ compared with $O_p\!\left(N_E^{-\gamma/(2\gamma+d)}\right)$ when applied directly to $X$. Thus the representation improves the rate governing CATE estimation from the experimental sample.

\paragraph{Identification under sufficiency violation.}\label{sec:suff_violation}
As noted in Remark~\ref{rem:testability}, we cannot explicitly enforce Assumption~\ref{assump:sufficiency}(ii) during training: since historical data contains only controls ($T=0$), this assumption may fail at $T=1$ when treatment effect heterogeneity on $S$ depends on features of $X$ not predictive of $S$ at baseline. We show that under our remaining assumptions, the estimator continues to recover a valid causal effect, and define the \emph{$\phi$-averaged CATE} as follows:
\[
\tau_\phi(x) := \mathbb{E}[Y^*(1) - Y^*(0) \mid \phi(X) = \phi(x), P=E].
\]
This is the average treatment effect among units sharing representation $\phi(x)$. It coincides with $\tau(x)$ when $\tau$ is constant on level sets of $\phi$. When the representation compresses away features that drive heterogeneity, $\tau_\phi(x)$ averages $\tau$ over those features. Thus heterogeneity along directions $\phi$ retains is preserved and only heterogeneity along compressed directions is averaged away.

\begin{theorem}[Identification under violation of Assumption~\ref{assump:sufficiency}(ii)]\label{thm:phi_cate_id}
Under Assumptions~\ref{assump:sutva}--\ref{assump:comparability} and Assumption~\ref{assump:sufficiency}(i),
\begin{equation}\label{eq:phi_cate}
\tau_\phi(x) = \mathbb{E}[h(S, \phi(x)) \mid \phi(x), T=1, P=E] - \mathbb{E}[h(S, \phi(x)) \mid \phi(x), T=0, P=E].
\end{equation}
\end{theorem}

The proof (Appendix~\ref{app:sufficiency-violation-proof}) closely follows the proof of Theorem~\ref{thm:cate_id}, with the sole distinction between Theorem~\ref{thm:cate_id} and Theorem~\ref{thm:phi_cate_id} being the invocation of Assumption~\ref{assump:sufficiency}(ii).

\paragraph{Bias-variance tradeoff.}
Theorem~\ref{thm:phi_cate_id} explains why the representation remains useful even under imperfect sufficiency. As is frequently seen in representation learning, compressing $X$ may discard some task-relevant information. Here, the compression replaces $\tau(x)$ with the coarser, but still causally interpretable $\tau_{\phi}(x)$. Relative to $\tau(x)$, this loss of heterogeneity is realized in bias due to compression, while estimation treatment effects over $\phi(X)\in\mathbb{R}^m$ rather than $X\in\mathbb{R}^d$ can reduce experimental-stage variance. 
Let $D$ denote the experimental sample and $\widehat{\tau}_{\phi,D}$ the resulting representation-based estimator. When this estimator is centered at $\tau_\phi$, its error satisfies
\[
\operatorname{MSE}(\widehat{\tau}_{\phi,D},\tau) = \underbrace{\mathbb{E}_X\!\left[\{\tau_\phi(X)-\tau(X)\}^2\right]}_{\text{squared bias due to compression}} +
\underbrace{\mathbb{E}_X\!\left[\operatorname{Var}_D\{\widehat{\tau}_{\phi,D}(\phi(X))\}\right]}_{\text{experimental-stage variance}}.
\]
The representation therefore reduces the mean squared error whenever the squared compression bias is smaller than the reduction in variance. This bias-variance tradeoff is especially relevant when $m\ll d$. We formalize this decomposition in Appendix~\ref{app:bias-variance}.

\subsection{Objective for obtaining \texorpdfstring{$\phi$}{phi}}\label{sec:objective_function}
Having established that $\phi(X)$ identifies the treatment effect, we now derive an objective for learning $\phi$ from historical data based on Information Bottleneck principles. The representation must satisfy both conditions of Assumption~\ref{assump:sufficiency}; we translate each into an optimization target.

\paragraph{{Condition} $(i)$: {Outcome-relevant information}.} The requirement $Y^{*} \perp X | \phi(X),S, P=H$, equivalent to $I(Y^{*}; X | \phi(X),S)= 0$, where $I(\cdot;\cdot)$ denotes mutual information (MI) \citep{cover2005elements}. By the chain rule of mutual information, 
\[
I(Y^*; X \mid S) = I(Y^*; \phi(X) \mid S) + I(Y^*; X \mid \phi(X), S).
\]

Since $I(Y^{*}; X | S)$ is constant with respect to $\phi$, minimizing $I(Y^*; X \mid \phi(X), S)$ is equivalent to maximizing $I(Y^*; \phi(X) \mid S)$:
\begin{align*}
\phi^* = \argmin_{\phi} I(Y^*; X \mid \phi(X), S) = \argmax_{\phi} I(Y^*; \phi(X) \mid S).
\end{align*}

\paragraph{{Condition} $(ii)$: {Surrogate-relevant information}.} The requirement $S \perp X \mid \phi(X), T, P=E$ poses a challenge, since we learn $\phi$ on historical data ($P=H$) where $T = 0$ for all units. We instead impose the proxy condition $S \perp X \mid \phi(X), P=H$, equivalent to $I(S; X \mid \phi(X)) = 0$, which we encourage by maximizing $I(S; \phi(X))$:
\begin{align}
\phi^* = \argmax_{\phi} I(S; \phi(X)).
\end{align}
Appendix~\ref{app:sufficiency-violation} characterizes the consequences when this proxy is imperfectly satisfied.

\textbf{Final objective:} We combine both terms as a finite-sample scalarization:
\begin{align}\label{eq:objective_function}
\phi^* = \arg\max_{\phi} \quad I(Y^{*}; \phi(X)| S) + \lambda \cdot I(S; \phi(X))
\end{align}
where $\lambda > 0$ reflects differences in noise levels between $Y^*$ and $S$. By the chain rule, $I((Y^*, S); \phi(X)) = I(S; \phi(X)) + I(Y^*; \phi(X) \mid S)$, so $\lambda = 1$ corresponds to maximizing the joint MI $I(\phi(X); (Y^*, S))$.

\paragraph{Practical implementation.}
We implement $\phi$ as a feedforward encoder mapping $X \in \mathbb{R}^d$ to $\mathbb{R}^m$. The conditional term $I(Y^*; \phi(X)\mid S)$ has no direct sample-based estimator, so we apply the chain rule above to rewrite Equation~\ref{eq:objective_function} as \[I(\phi(X); (Y^*, S)) + (\lambda - 1)\cdot I(\phi(X); S)\] 
which decomposes the objective into two unconditional MI terms that admit standard variational lower bounds. We train the encoder jointly with two MI estimator heads: one for joint pair $(Y^*, S)$, and one for $S$ alone, instantiated as either InfoNCE \citep{oord2018representation} or MINE \citep{belghazi2018mutual}. As a simpler alternative, we can replace the MI estimators with prediction functions from $\phi(X)\to\hat S$ and $(\phi(X), S)\to\hat Y^*$ trained with BCE (binary) or MSE (continuous) losses. The prediction losses lower-bound the corresponding MI terms (see Appendix~\ref{app:mi_bounds}), making this substitution theoretically valid. We report results across all three estimators in Section~\ref{sec:experiments}.

\section{Empirical Results}\label{sec:experiments}
In this section, we provide empirical validation for our claims that using representations that explicitly capture the predictive relationship between $X$ and $S$ produces more sample-efficient estimates of the CATE than other standard methods. We provide our \href{https://github.com/the-maitrix/crl-hte/}{code (linked)}.

\paragraph{Baselines}
We compare our method against baselines representative of standard techniques employed by practitioners, who typically either use raw covariates or apply {off-the-shelf} dimensionality reduction on the covariates before CATE estimation. We consider three categories of baselines:

\begin{enumerate}
    \item \textbf{Raw covariates $X$}: using the high dimensional covariates $X$ to fit the CATE learner. The most straightforward approach in which no dimensionality reduction is applied.
    \item \textbf{Unsupervised reduction (PCA, ICA, Autoencoder)}: Principal Component Analysis \cite{pearson1901pca}, Independent Component Analysis \cite{jutten1991blind,hyvarinen2000independent}, and autoencoders \cite{baldi1989neural,kramer1991nonlinear} are widely adopted methods for handling high-dimensional covariates without requiring outcome supervision.
    \item \textbf{Supervised reduction (PLS)}:  Partial Least Squares \cite{wold1966estimation}, a supervised dimensionality reduction of $X$ using $S$ as targets. PLS has access to the same surrogate information as our method, serving as a near-linear analog of our method's surrogate-relevant objective; under joint Gaussianity the two recover related subspaces. 
\end{enumerate}

We use the same latent representation dimension across all dimensionality-reduction-based methods (ours and the baselines). Further implementation details are provided in Appendix ~\ref{app:empirical_extra}. 

\paragraph{Our method}
We evaluate three variants of our representation-learning objective. MI-MINE and MI-InfoNCE optimize conditional mutual-information objectives using MINE \citep{belghazi2018mutual} and InfoNCE \citep{oord2018representation} respectively, with the joint target $(Y,S)$. The prediction encoder learns $\phi(X)$ using prediction heads for $Y$ and $S$. All three methods map $X$ to an $m$-dimensional representation and otherwise use the same outcome-prediction and CATE-estimation pipeline. Architecture and other implementation details are provided in Appendix~\ref{app:empirical_extra}.

\paragraph{Procedure}
All methods follow a two-stage procedure. In the first stage, we use the historical sample to (i) learn a representation (where applicable), and (ii) train an outcome model. For all methods, the outcome model $h:\phi(X)\times S\to Y$ maps representations and surrogates to predicted outcomes $\hat{Y}$. For the raw-covariate baseline, we additionally fit $h(X)$ with no surrogate information. In the second stage, we apply $h$ to the experimental sample and use $\hat{Y}$ as the CATE target. All representation-based methods fit the CATE learner conditioned on the representation, while the raw-X baselines use $X$ directly in the CATE learner.
We use cross-fitted DML learners for CATE estimation. Further implementation details are in Appendix~\ref{app:empirical_extra}.

\paragraph{Metrics}
We report mean normalized PEHE \cite{hill2011bayesian,shalit2017estimating} and normalized top-20\% policy value (for which $0$ is random targeting and $1$ is oracle targeting). We repeat each experiment on 10 random data draws, using the same draws for every method.

\begin{figure*}[t]
    \centering
    \includegraphics[width=\textwidth]{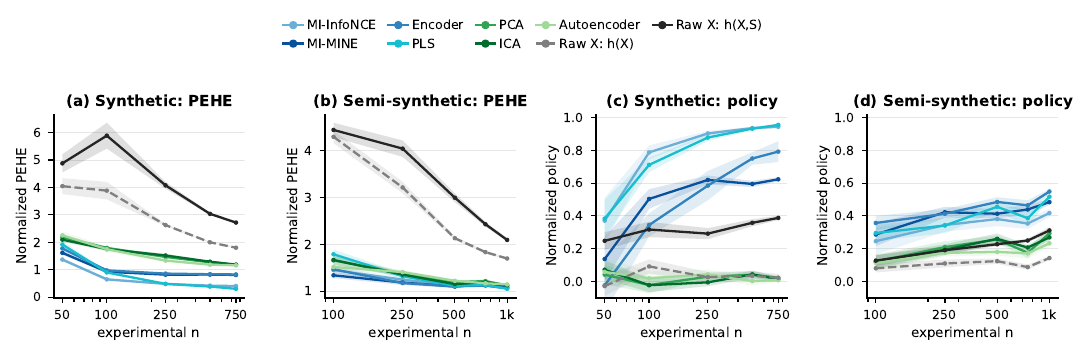}
    \caption{Normalized PEHE and normalized top-20\% policy value on the synthetic and perinatal semi-synthetic datasets at $\dim(\phi)=10$. Curves show means over 10 runs and the surrounding shaded regions show $\pm 1$ standard error. The two raw-$X$ curves differ only in whether the historical outcome model uses $S$.}
    \label{fig:main_results}
    \label{fig:synth_PEHE}
    \label{fig:upmc_PEHE}
    \label{fig:synth_policy}
    \label{fig:semisynth_policy}
\end{figure*}

\subsection{Synthetic dataset}\label{sec:exp_synth}

\paragraph{Description.}
We draw covariates $X \sim \mathcal{N}(\mathbf{0},\mathbf{I}_d)$ with $d=1000$. We define the ground-truth latent as a 10-dimensional vector $Z = X W^{\top}$, where each coordinate of $Z$ depends on a small subset of covariates through the sparse matrix $W$. We split the latent into two halves, $Z = (Z_s, Z_y)$ with $Z_s, Z_y \in \mathbb{R}^5$. In the main experiment, $Z_s$ determines the surrogates $S \in \mathbb{R}^{10}$ and the treatment effect on $S$, and we generate $Y$ as a noisy linear function of $S$. Since $Y$ and the treatment effect depend on $X$ only through $Z_s$, our sufficiency assumption holds exactly with a five-dimensional representation. 

In Appendices~\ref{app:suff-violation-empirical} and~\ref{app:delta-violation-empirical}, we study what happens when our assumptions do not hold by letting $Z_y$ influence the outcome and treatment effect. Further implementation details are in Appendix~\ref{app:synth_exp_extra}.

\paragraph{Experiment design.}
We generate $N_H = 10{,}000$ historical samples with $T \equiv 0$ and an experimental pool of $N_E = 5{,}000$ with $T \sim \mathrm{Bernoulli}(0.5)$. Encoders are trained on $H$ and CATE estimation is run on stratified subsamples at $n \in \{50,100,250,500,750\}$. We set $\dim(\phi)=10$ and use the predicted outcome $\widehat Y$ as the CATE target. Appendix~\ref{app:m-sweep-empirical} examines sensitivity to the choice of the bottleneck dimension by varying $\dim(\phi)\in{2,3,5,10,20,50}$ for a fixed DGP.
We report results averaged over 10 historical and experimental resamples from the same fixed DGP.

\paragraph{Results.}
Figure~\ref{fig:main_results}(a) reports normalized PEHE against $n$ for the DML learner. The supervised methods (ours and PLS) clearly outperform the unsupervised methods across all sample sizes, and both raw-$X$ variants have the worst performance: at $n=100$, MI-InfoNCE achieves normalized PEHE $0.650\pm0.059$, compared with $0.902\pm0.047$ for PLS, at least $1.750$ for the unsupervised baselines, and $3.883\pm0.319$ for raw $X$. Adding $S$ to the raw-$X$ outcome model does not help ($5.886\pm0.471$). The policy metric shows a similar small-sample trend: at $n=100$, MI-InfoNCE attains a top-20\% policy value of $0.788\pm0.043$, versus $0.091\pm0.043$ for raw $X$ with $h(X)$.

The linear DGP is well suited to PLS, which is accordingly the strongest baseline: MI-InfoNCE and PLS are effectively tied by $n=250$, and PLS is modestly ahead at larger $n$. That our method leads at the smallest samples and substantially outperforms every unsupervised baseline throughout confirms the importance of the outcome and surrogate relevant objective, while all lower dimensional methods outperforming both variants of the raw-$X$ baseline highlight the contribution of dimensionality reduction. Tabulated results across all sample sizes, the X-learner, and the true-experimental-$Y$ variant are in Appendix~\ref{app:synth_exp_extra} (Tables~\ref{tab:synth_pehe_norm_xl}--\ref{tab:synth_policy_norm_20_dml}).

\subsection{Semi-synthetic medical dataset}\label{sec:exp_semisynth}
\paragraph{Description.}
We construct a semi-synthetic dataset using real covariates from a perinatal depression study at a large academic medical center (anonymized), using real covariates and diagnosis trajectories with synthetic treatment assignment and potential outcomes. More details in Appendix~\ref{app:semi_synthetic}.

\paragraph{Cohorts and covariates.}
The historical cohort comprises $N_H=60{,}248$ pregnancies, and the experimental pool contains $N_E=11{,}747$ pregnancies from a perinatal mobile-app study. After preprocessing, both cohorts share $X\in\mathbb{R}^{340}$. The covariates include maternal demographics, obstetric history and pregnancy characteristics, health behaviors, clinical measurements, diagnoses during pregnancy, and health-care utilization. Further details are provided in Appendix~\ref{app:semi_synthetic}.

\paragraph{Auxiliary outcomes.}
For the auxiliary outcomes, we use diagnoses recorded in the EHR during pregnancy for conditions related to perinatal depression risk such as anxiety, depression, bipolar disorder, obsessive-compulsive disorder, trauma reactions, and substance-use disorders, along with conditions such as hypertension, diabetes, and autoimmune, cardiac, kidney, and liver conditions. Each condition is recorded at multiple points during pregnancy, giving 56 binary indicators; we drop those that are nearly always zero or one in the historical cohort, leaving 37 features. We summarize them with their first five principal components, $S\in\mathbb{R}^{5}$, standardizing and fitting the PCA on the historical cohort only. All of these diagnoses are observed in both cohorts.

\paragraph{Synthetic treatment and outcome.}
Treatment assignment is randomized with probability $1/2$ in the experimental sample. The synthetic treatment shifts $S$ along an outcome-relevant direction whose magnitude depends on $X$. Potential outcomes are then generated from the same outcome model $P(Y\mid X,S)$ in both cohorts, so there is no direct $T\to Y$ path. This makes randomization, surrogacy, and outcome-model comparability hold by construction. Appendix~\ref{app:semisynth_assumptions} assesses how well our assumptions hold on the semi-synthetic dataset.

\paragraph{Experiment design.}
We set $\dim(\phi)=10$ and train the representations and outcome models on $10{,}000$ historical observations. For each run, we reserve 5,000 experimental instances for testing and draw training samples of size $n\in\{100,250,500,750,1000\}$ from the remaining pool. We estimate the CATE using three-fold cross-fitted DML, replacing the experimental outcome with its prediction $\widehat Y$. We repeat the experiment over 10 random data draws and use the same draws for every method. Figure~\ref{fig:main_results}(b,d) compares all methods, including raw-$X$ baselines using $h(X)$ and $h(X,S)$. Tabulated results are in Appendix~\ref{app:semi_synthetic}, along with additional results for $\dim(\phi)=5$, and runs using the true experimental outcome, along with the remaining implementation details.

\paragraph{Results.}
As in the synthetic dataset, Figure~\ref{fig:main_results}(b,d) shows a clear advantage of dimension-reduction methods. At $n=100$, conditional MINE achieves a normalized PEHE of $1.340$, compared with more than $4.2$ for either raw-$X$ baseline. The gap narrows as the experimental sample grows but remains substantial: at $n=1000$, the prediction encoder and PLS both achieve approximately $1.06$, compared with $1.70$ for $h(X)$ and $2.10$ for $h(X,S)$. The policy results show a similar trend in relative performance. Appendix~\ref{app:semi_synthetic} reports more detailed results.

\section{Discussion}\label{sec:conclusion}
\paragraph{Contributions}
This paper aims to help ML practitioners in healthcare, policy, and industry who rely on RCTs and A/B tests to estimate heterogeneous effects of interventions but are limited by small experimental sample sizes. These settings often come with access to rich historical data: EHRs, administrative records, prior A/B tests. However, the potential of such data to aid causal inference is often overlooked. We provide theoretical justification, empirical evidence, and a systematic way to use representation learning to leverage such data, and to improve the statistical power of CATE estimation in small-scale experiments. We hope that this also motivates more deliberate collection of auxiliary outcomes in administrative and operational data, so that the structural signal needed for sample-efficient causal inference is available when new interventions are tested.

\paragraph{Limitations and future work}
Our identification result relies on surrogacy and sufficiency, and while our analysis shows that empirical gains are robust to imperfect satisfaction of these conditions -- sufficiency violations cost only granularity (the estimand coarsens to the average effect among units sharing $\phi(X)$), and surrogacy violations are sidestepped whenever $Y^*$ is observed in the experiment -- characterizing performance under more general failure modes remains open. Practitioners will also benefit from extensions of our work, such as fine-tuning $\phi$ on experimental data, allowing treated units in the historical sample, swapping the MLP encoder for pretrained image or text encoders, supporting multiple treatments and target outcomes, and extending to observational studies.

\bibliography{references}
\bibliographystyle{plainnat}

\newpage
\appendix
\crefname{section}{Appendix}{Appendices}
\Crefname{section}{Appendix}{Appendices}
\onecolumn
\section{Proof of \Cref{thm:cate_id}}\label{app:cate_id_proof} 
In this section we prove that under Assumptions~\ref{assump:sutva}-\ref{assump:sufficiency}, the CATE $\tau(x) = \mathbb{E}[Y^{*}(1) - Y^{*}(0) \mid X = x, P = E]$ is identified by Equation~\eqref{eq:our_cate_estimator}. Since $\tau(x)$ is a difference of conditional expectations, it suffices to show that for each treatment arm $t \in \{0,1\}$:
\[
\mathbb{E}[Y^{*}(t) \mid X = x, P = E] = \mathbb{E}[h(S, \phi(x)) \mid \phi(x), T = t, P = E]
\]
where $h(s,z) = \mathbb{E}[Y^{*} \mid S = s, \phi(X) = z, P = H]$. For brevity, we abbreviate $Y^*$ as $Y$.
\begin{proof}
We prove Equation~\eqref{eq:our_cate_estimator} for each treatment arm $t \in \{0,1\}$.
\begin{align*}
\mathbb{E}[Y(t) \mid X=x, P=E] 
    &= \mathbb{E}[Y \mid X=x, T=t, P=E] 
    && \text{(ignorability, SUTVA)}\\
    &= \mathbb{E}\big[\mathbb{E}[Y \mid S, T=t, X=x, P=E] \,\big|\, X=x, T=t, P=E\big] 
    && \text{(iterated expectations)}\\
    &= \mathbb{E}\big[\mathbb{E}[Y \mid S, X=x, P=E] \,\big|\, X=x, T=t, P=E\big] 
    && \text{(surrogacy, A\ref{assump:surrogacy})}\\
    &= \mathbb{E}\big[\mathbb{E}[Y \mid S, X=x, P=H] \,\big|\, X=x, T=t, P=E\big] 
    && \text{(comparability, A\ref{assump:comparability})}\\
    &= \mathbb{E}\big[\mathbb{E}[Y \mid S, \phi(x), P=H] \,\big|\, X=x, T=t, P=E\big] 
    && \text{(sufficiency (i), A\ref{assump:sufficiency})}\\
    &= \mathbb{E}\big[h(S, \phi(x)) \,\big|\, X=x, T=t, P=E\big] 
    && \text{(definition of $h$)}\\
    &= \mathbb{E}\big[h(S, \phi(x)) \,\big|\, \phi(x), T=t, P=E\big] 
    && \text{(sufficiency (ii), A\ref{assump:sufficiency})}.
\end{align*}
Taking the difference across $t \in \{0,1\}$ yields Equation~\eqref{eq:our_cate_estimator}.
\end{proof}

\subsection{Discussion of proof steps}
\paragraph{Step 1 (Ignorability).} By randomization of treatment assignment in the experimental sample, Assumptions~\ref{assump:sutva}--~\ref{assump:unconfounded} ensures that the potential outcome $Y(t)$ equals the observed outcome $Y$ when $T = t$. This is a standard identification step in any randomized experiment.

\paragraph{Step 2 (Tower property).} We introduce $S$ via the law of iterated expectations.

\paragraph{Step 3 (Surrogacy).} Assumption~\ref{assump:surrogacy} states $T \perp Y \mid S, X, P = E$. Treatment is conditionally independent of the outcome given surrogates and covariates, so we drop $T$ from the inner expectation. The outer expectation retains $T$, which governs the distribution of $S$.

\paragraph{Step 4 (Comparability).} Assumption~\ref{assump:comparability} states $Y \perp P \mid S, X$. The conditional distribution of $Y$ given $(S, X)$ is invariant across populations, allowing us to switch from $P=E$ to $P=H$ in the inner expectation.

\paragraph{Step 5 (Sufficiency (i)).} Assumption~\ref{assump:sufficiency}(i) states $Y \perp X \mid \phi(X), S, P = H$. Conditional on the representation $\phi(X)$ and surrogates, the full covariates are redundant for predicting $Y$.

\paragraph{Step 6 (Outcome model).} We define $h(s, z) = \mathbb{E}[Y \mid S = s, \phi(X) = z, P = H]$.

\paragraph{Step 7 (Sufficiency (ii)).} Assumption~\ref{assump:sufficiency}(ii) states $S \perp X \mid \phi(X), T, P = E$. Conditional on the representation and treatment, the full covariates are redundant for predicting $S$, so the outer expectation depends on $X$ only through $\phi(X)$.

\paragraph{Combining.}
We obtain:
\[\mathbb{E}[Y(t) \mid X = x, P = E] = \mathbb{E}[h(S, \phi(x)) \mid \phi(x), T = t, P = E]\]
The CATE is:
\[\tau(x) = \mathbb{E}[h(S, \phi(x)) \mid \phi(x), T = 1, P = E] - \mathbb{E}[h(S, \phi(x)) \mid \phi(x), T = 0, P = E]\]
\paragraph{Summary.} 
 Our main result specifies what information the representation must retain from the covariates, and what can be ignored (Steps 5 and 7). It also combines historical and experimental data in the following way: historical data is used to fit the outcome model $h$, which captures how the surrogates and baseline covariates jointly predict the outcome. Experimental data is used to estimate how treatment shifts the distribution of surrogates conditional on the representation $\phi(x)$. Finally, surrogacy ensures that treatment affects $Y$ only through $S$ and, along with comparability, thus allows us to use the historical outcome model to estimate the CATE in the experimental sample. 

\section{On ignorability assumption}\label{app:ignorability-amin}
\subsection{Ignorability under compression}\label{app:ignorability}

We formalize and prove the claim that in randomized experiments, ignorability (Assumption~\ref{assump:unconfounded}) is preserved by any function of the covariates -- including any learned representation $\phi(X)$. For brevity, we abbreviate $Y^*$ as $Y$.

\begin{lemma}[Ignorability under compression]\label{lem:ignorability}
Under Assumption~\ref{assump:unconfounded}, for any measurable $\phi: \mathbb{R}^d \to \mathbb{R}^m$,
\[
T \perp (Y(0), Y(1)) \mid \phi(X), P=E.
\]
\end{lemma}
\begin{proof}
Assumption~\ref{assump:unconfounded} gives $T \perp (X, Y(0), Y(1)) \mid P=E$. Since $\phi$ is a measurable function of $X$, this implies $T \perp (\phi(X),Y(0),Y(1)) \mid P=E$. By the Weak Union axiom for conditional independence \citep{dawid1979conditional}, $T \perp (Y(0), Y(1)) \mid \phi(X), P=E$ for any measurable $\phi$.
\end{proof}

This contrasts with observational representation learning \citep{johansson2016learning, shalit2017estimating}, where ignorability holds conditional on $X$, and a learned $\phi$ must explicitly retain the variables responsible for that conditioning. In our setting, randomization makes treatment unconditionally independent of all covariates and potential outcomes, so any function of $X$ (including $\phi(X)$) preserves ignorability. The representation can therefore be optimized for information relevant to the target outcome and surrogates (Section~\ref{sec:objective_function}), without a separate requirement to retain information for confounding adjustment.

We discuss below the extension to stratified randomization, where ignorability holds only conditional on a known stratification subset $V \subset X$.

\subsection{Extension to stratified randomizations}\label{app:stratified}
Assumption~\ref{assump:unconfounded} corresponds to simple randomization, as seen in standard RCTs. Under stratified randomization, treatment is assigned independently within strata defined by a known subset of covariates $V\subset X$. Within strata, treatment is independent of all other covariates and potential outcomes:
\begin{equation}\label{eq:joint_ignorability}
    T\perp (X, Y(1),Y(0))\mid V, P=E
\end{equation}
Since $V$ contains all covariates governing treatment assignment in a stratified RCT, treatment is independent of all other covariates and potential outcomes given $V$. Thus we can define: 
\[
\tilde{\phi}(X):= (V,\phi(X))
\]
where $\phi$ is trained on the full $X$ using the objective in Equation~\ref{eq:objective_function}. Concatenating $V$ ensures it is preserved regardless of whether the encoder retains it. Because $\phi(X)$ is a function of $X$, \eqref{eq:joint_ignorability} implies $T\perp (\phi(X),Y(1), Y(0)) \mid V,P=E$. Applying Weak Union axiom (as in Lemma~\ref{lem:ignorability}) yields $T\perp (Y(1), Y(0)) \mid V,\phi(X), P=E$, i.e. ignorability is preserved: $T \perp (Y(0), Y(1)) \mid \tilde{\phi}(X), P=E$. Theorem~\ref{thm:cate_id} applies with $\tilde\phi$ in place of $\phi$, provided Assumption~\ref{assump:sufficiency} is restated with respect to $\tilde{\phi}(X)$. Since $\tilde\phi(X)$ contains strictly more information than $\phi(X)$, the sufficiency conditions are no harder to satisfy than in the unstratified case. The proof of Theorem~\ref{thm:cate_id} is then identical with $\tilde\phi$ substituted throughout. The dimensional benefit is preserved: $\tilde{\phi}(X) \in \mathbb{R}^{m + |V|}$, and stratification variables are typically low-dimensional, so $m + |V| \ll d$.

\section{Sufficiency Violation}\label{app:sufficiency-violation}
\subsection{Identification under violation}\label{app:sufficiency-violation-proof}
In this section, we analyze the estimator from Theorem~\ref{thm:cate_id} when Assumption ~\ref{assump:sufficiency}(ii) is violated. Recall that (ii) requires $S \perp X \mid \phi(X), T, P=E$ for all treatment levels $t \in \{0,1\}$. Since the representation $\phi$ is learned on historical data where $T=0$ for all units, the proxy objective in Section ~\ref{sec:objective_function} can at best enforce this condition at $T=0$; it may fail at $T=1$ when treatment effect heterogeneity on $S$ depends on features of $X$ not predictive of $S$ at baseline.

We define the $\phi$-averaged CATE as:
\[
\tau_{\phi}(x):=\E[Y^*(1) - Y^*(0) \mid \phi(X) = \phi(x), P=E].
\]
This is the average treatment effect among units sharing representation $\phi(x)$. It coincides with $\tau(x)$ when $\tau$ is constant on level sets of $\phi$ (eg. trivially for bijective $\phi$). When the representation compresses away features that drive treatment effect heterogeneity, $\tau_{\phi}(x)$ averages $\tau$ over those features. That is, heterogeneity along directions $\phi$ retains is preserved, only heterogeneity along compressed directions is averaged over. For brevity, we shall abbreviate $Y^*$ as $Y$ in the following analysis.

From Theorem~\ref{thm:phi_cate_id}, under Assumptions~\ref{assump:sutva}--\ref{assump:comparability} and Assumption~\ref{assump:sufficiency}(i),
    \[\tau_{\phi}(x) = \E[h(S,\phi(x))\mid \phi(x), T=1, P=E]- \E[h(S,\phi(x))\mid \phi(x), T=0, P=E]\]

\begin{proof}
    For each treatment arm $t\in\{0,1\}$, we show
    \[
    \E[h(S, \phi(x)) \mid \phi(X) = \phi(x), T=t, P=E] = \E[Y(t) \mid \phi(X) = \phi(x), P=E].
    \]
    By the law of iterated expectations, we can expand $\E[h(S, \phi(x)) \mid \phi(X) = \phi(x), T=t, P=E]$ as
    \[\E\big[ \E[h(S, \phi(x)) \mid X, T=t, P=E]\big| \phi(X) = \phi(x), T=t, P=E\big]\].
    The outer expectation integrates over $X$ on the set $\{x:\phi(x')=\phi(x)\}$. On this set, steps 1--6 of the proof of Theorem~\ref{thm:cate_id} (which do not use Assumption~\ref{assump:sufficiency}(ii)) apply with $\phi(x')=\phi(x)$, giving
    \[
    \E[h(S, \phi(x)) \mid X = x', T=t, P=E] = \E[Y(t) \mid X = x', P=E].
    \]
    By Assumption~\ref{assump:unconfounded}, $T \perp X \mid P=E$, so the distribution of $X$ given $\phi(X) = \phi(x)$ is the same with or without conditioning on $T=t$. Therefore
    $\begin{aligned}
    \E[h(S, \phi(x)) \mid \phi(X) = \phi(x), T=t, P=E] = &\E[\E[Y(t) \mid X, P=E] \mid \phi(X) = \phi(x), P=E] \\ &= \E[Y(t) \mid \phi(X) = \phi(x), P=E].
    \end{aligned}$
    Taking the difference between $t \in \{0,1\}$ yields $\tau_\phi(x)$.
\end{proof}

\subsection{Bias-variance tradeoff under imperfect sufficiency}\label{app:bias-variance}

We now characterize the bias-variance tradeoff discussed in Section~\ref{sec:main-result}. We note that we are concerned with setting in which $N_H \gg N_E$. Our target is the original CATE $\tau(x)$. When Assumption~\ref{assump:sufficiency} does not hold perfectly, Theorem~\ref{thm:phi_cate_id} instead identifies 
\[
\tau_\phi(x):=\mathbb{E}[Y^*(1)-Y^*(0)\mid \phi(X)=\phi(x),P=E].
\]
By the law of iterated expectations, we have:
\[
\tau_\phi(X)=\mathbb{E}[\tau(X)\mid\phi(X),P=E].
\]

Given an experimental sample $D$ of size $N_E$, let $\widehat{\tau}_{\phi,D}(\phi(x))$ denote our representation-based estimator. Since the representation and outcome model are learned
from the much larger historical sample, we treat them as fixed and focus on the error arising in the experimental stage. We further assume that the experimental estimator is
centered at the coarsened CATE:
\[
\mathbb{E}_D\left[\widehat{\tau}_{\phi,D}(\phi(x))\right]=\tau_\phi(x).
\]

Our error thus lends itself to a familiar bias-variance decomposition,
\[
\mathbb{E}_{X,D}\left[\left\{\widehat{\tau}_{\phi,D}(\phi(X))-\tau(X)\right\}^2\right]
=
\underbrace{\mathbb{E}_X\left[\left\{\tau_\phi(X)-\tau(X)\right\}^2\right]}_{\text{squared bias due to coarsening}}
+
\underbrace{\mathbb{E}_X\left[\operatorname{Var}_D\left\{\widehat{\tau}_{\phi,D}(\phi(X))\right\}\right]}_{\text{variance due to the finite experimental sample}}.
\]
The cross term is zero because
\[
\mathbb{E}\left[\tau(X)-\tau_\phi(X)\mid\phi(X),P=E\right]=0.
\]

For comparison, an estimator that uses raw covariates directly would have no coarsening bias (assuming it is an unbiased estimator), and the error would consist entirely of variance due to the finite experimental sample. Let $V_X$ and $V_{\phi}$ denote the expected variances of the raw-covariate and representation-based estimators respectively. The representation-based estimator has lower mean squared error whenever 
\[
\mathbb{E}_X\left[\left\{\tau_\phi(X)-\tau(X)\right\}^2\right] < V_X(N_E)-V_\phi(N_E).
\]
Thus we still benefit from the representation-based estimator, even under imperfect sufficiency. The representation reduces the total error whenever the reduction in variance exceeds the squared bias due to coarsening. In the following section, we empirically examine how the tradeoff impacts performance of our method as we increase the extent of sufficiency violation.

\subsection{Empirical analysis of sufficiency violation}\label{app:suff-violation-empirical}
We complement Theorem~\ref{thm:phi_cate_id} with an empirical sweep over the degree of sufficiency-(ii) violation. Building on the synthetic DGP of Section~\ref{sec:exp_synth}, we introduce a scalar $\alpha\in[0,1]$ that controls how much of the treatment-induced shift in $S$ is driven by features of $X$ \emph{not} predictive of $S$ at baseline ($T=0$):
\[
  S_i(t) = S_i(0) + t\cdot\left\{\gamma_0 + \sqrt{1-\alpha^2}\,g_s(Z_{s,i}) + \alpha\,g_y(Z_{y,i})\right\},
\]
The baseline surrogate $S_i(0)$ depends on $Z_s$ but not on $Z_y$. Thus, the historical control data contain information about $Z_s$, but provide no signal for the representation to retain $Z_y$. Thus, $g_y(Z_y)$ introduces treatment-effect heterogeneity that cannot be learned from the untreated historical sample. We scale $g_s(Z_s)$ and $g_y(Z_y)$ to have equal variance, so that changing $\alpha$ changes the source of the heterogeneity without changing its overall magnitude. At $\alpha=0$, Assumption~\ref{assump:sufficiency}(ii) holds and we identify $\tau(x)$. As $\alpha$ increases, more heterogeneity depends on features that $\phi$ cannot recover from the historical data, and Theorem~\ref{thm:phi_cate_id} instead identifies $\tau_\phi(x)$. At $\alpha=1$, all treatment-effect heterogeneity is determined by these features.

We use the same three-fold DML learner as in Section~\ref{sec:experiments}. We run each method using both the historical outcome-model prediction $\widehat Y$ and the true experimental outcome $Y$ as the CATE target. We vary $\alpha\in\{0,0.25,0.5,0.75,1\}$ and $n\in\{50,100,250,500,750\}$ and report mean $\pm$ one standard error over the same ten paired historical and experimental samples used in Section~\ref{sec:exp_synth}. Remaining implementation details in Appendix~\ref{app:synth_exp_extra}.

\begin{figure}[t]
  \centering
  \begin{subfigure}[t]{\textwidth}\centering
    \includegraphics[width=\textwidth]{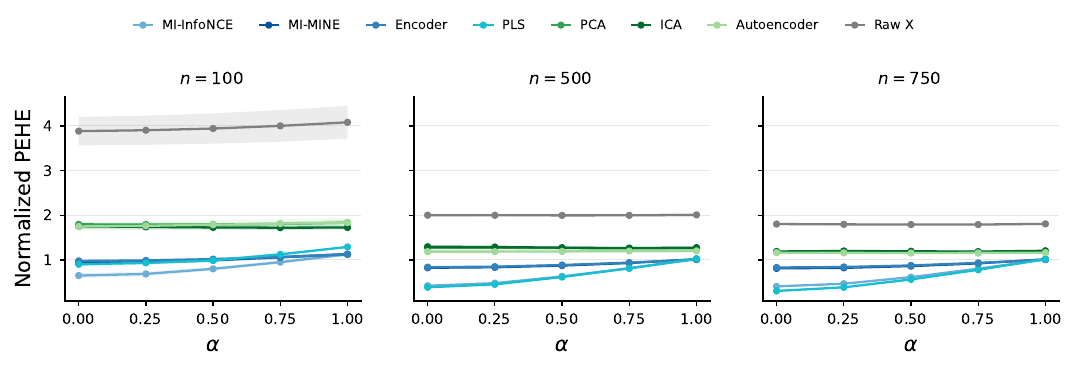}
    \caption{Normalized PEHE, $\hat Y$ target.}
  \end{subfigure}\\[2pt]
  \begin{subfigure}[t]{\textwidth}\centering
    \includegraphics[width=\textwidth]{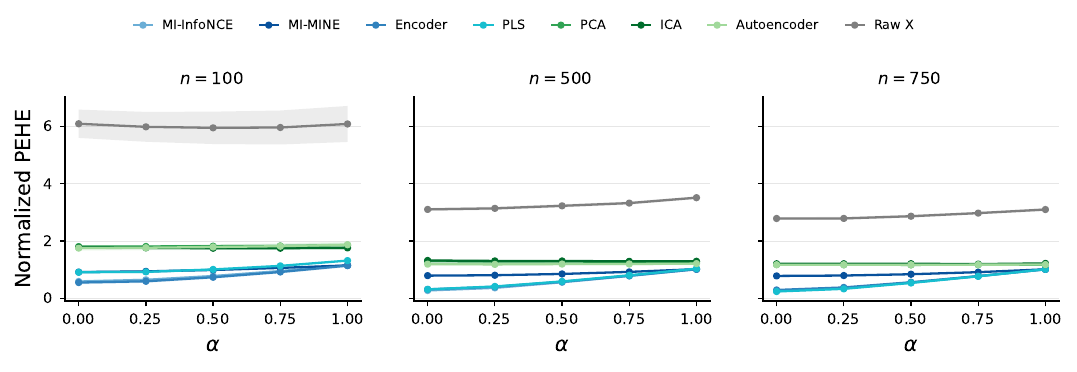}
    \caption{Normalized PEHE, $Y$ target.}
  \end{subfigure}
  \caption{\textbf{Sufficiency-(ii) violation, DML --- normalized PEHE.}
True-outcome-aware representations retain a clear advantage under partial violations, but their normalized PEHE increases with $\alpha$ as more treatment-effect heterogeneity depends on features that cannot be learned from the historical data.}
  \label{fig:alpha-sweep-dml-pehe}
\end{figure}

\begin{figure}[t]
  \centering
  \begin{subfigure}[t]{\textwidth}\centering
    \includegraphics[width=\textwidth]{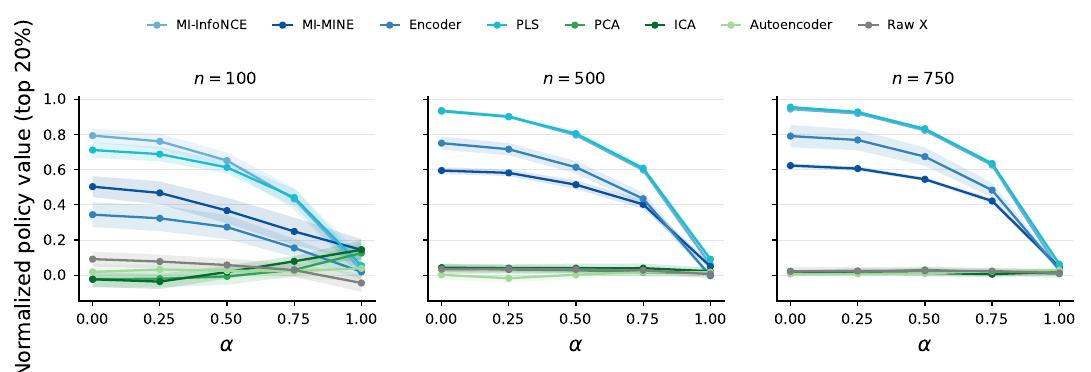}
    \caption{Normalized policy value (top 20\%), $\hat Y$ target.}
  \end{subfigure}\\[2pt]
  \begin{subfigure}[t]{\textwidth}\centering
    \includegraphics[width=\textwidth]{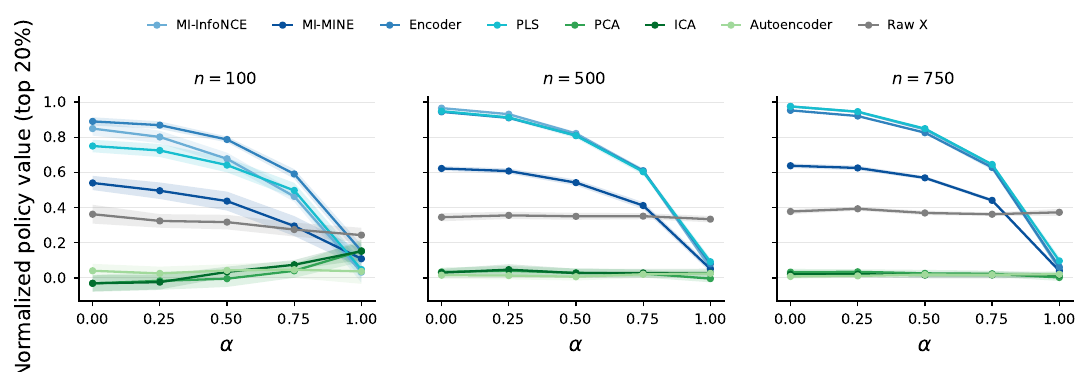}
    \caption{Normalized policy value (top 20\%), $Y$ target.}
  \end{subfigure}
\caption{\textbf{Sufficiency-(ii) violation, DML --- normalized
    top-20\% policy value.} The policy advantage declines gradually under partial violation and disappears at $\alpha=1$, when the representation no longer contains the information that determines treatment-effect rankings.}
  \label{fig:alpha-sweep-dml-policy}
\end{figure}

\paragraph{Results.}
Figures~\ref{fig:alpha-sweep-dml-pehe} and~\ref{fig:alpha-sweep-dml-policy} report normalized PEHE and policy value across $\alpha$, respectively.
At $n=500$, increasing $\alpha$ from zero to $0.5$ raises MI-InfoNCE's normalized PEHE from $0.423$ to $0.549$, while its policy value remains high ($0.936$ to $0.862$). At $\alpha=1$, however, its PEHE reaches $1.031$ and its policy value falls to $0.041$; Other supervised methods follow the same pattern, and the trend is similar across sample sizes. We observe that compression can still reduce squared error relative to raw $X$, whose normalized PEHE is approximately $2.0$ in this setting, but it cannot recover heterogeneity absent from the representation, as seen at $\alpha = 1.0$. The results support our method's robustness to significant violations of sufficiency, although not to complete violations.

\section{Surrogacy Violation}\label{app:surrogacy-violation}

Appendix~\ref{app:suff-violation-empirical} characterizes what happens to the estimator when sufficiency-(ii) is violated. Here we treat the second identification assumption that may not hold exactly in practice: \emph{surrogacy} (Assumption~\ref{assump:surrogacy}), which requires $T\!\perp\!Y^*\mid S, X, P=E$ -- i.e., that the surrogates
$S$ fully mediate the effect of treatment on the target outcome. When this fails, treatment exerts an additional direct effect on $Y^*$ that bypasses $S$.

Figure~\ref{fig:dags} illustrates both violations within the broader graphical model assumed by our framework. The top row shows the unviolated case: the historical DAG has no treatment node, the experimental DAG introduces $T$ acting on $S$ (with $S$ in turn driving $Y^*$), and the latent $Z$ affects both $S$ and $Y^*$ throughout. The bottom row depicts the two violations we are concerned with. Surrogacy violation (left) adds a direct $T\!\to\!Y^*$ edge, breaking the mediation through $S$. Sufficiency-(ii) violation (right) is treatment-arm-specific: at $T=0$ the structure matches what $\phi$ was trained on, while at $T=1$ a violating $X\!\to\!S$ edge appears, encoding heterogeneity in the treatment effect on $S$ that $\phi$ cannot recover from historical (untreated) data. The two violations affect CATE estimation in different ways: sufficiency-(ii) coarsens the estimand to $\tau_\phi(x)$ (Theorem~\ref{thm:phi_cate_id}), while surrogacy violation introduces bias unless $Y^*$ is observed in the experimental sample, in which case randomization identifies the representation-conditional CATE directly.

In the following section, we run an empirical sweep over a scalar $\delta$ controlling the strength of the direct $T\!\to\!Y^*$ path, and study the robustness of our method to such a violation.

\begin{figure}
    \centering  
    \includegraphics[width=1\linewidth]{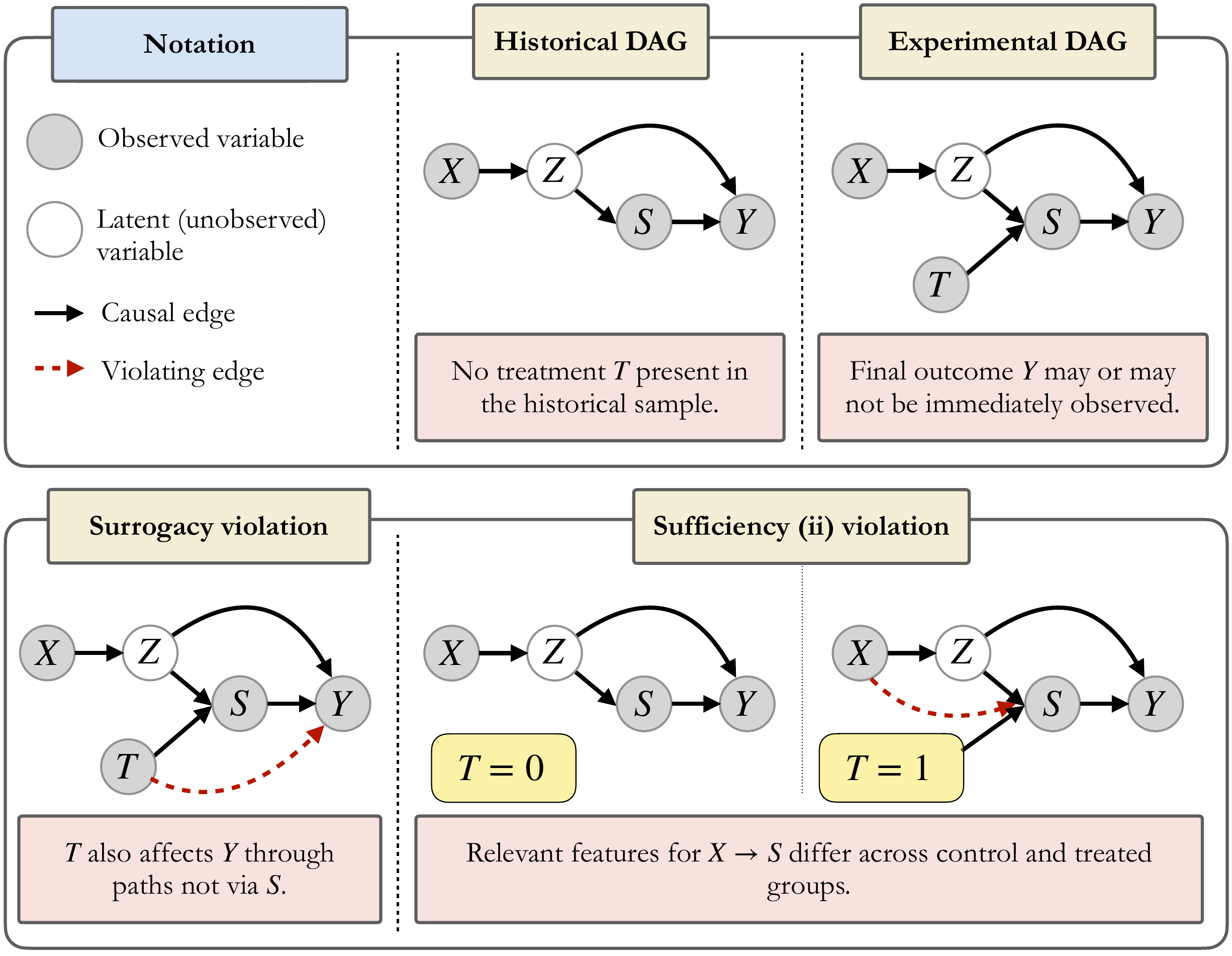}
    \caption{\textbf{Causal structure: default and two identification violations.} \emph{Top row, left to right:} notation; the historical DAG, where no treatment is present; the  experimental DAG, where $T$ acts on $S$ and $S$ mediates the effect on $Y^*$. $Z$ is the ground-truth latent representation of $X$ on which $S$ and $Y^*$ depend; $\phi(X)$ is trained to recover $Z$. \emph{Bottom row:} the two violations. \emph{Left, surrogacy violation:} a direct $T\!\to\!Y^*$ edge (dashed red) breaks the mediation through $S$; analyzed in Appendix~\ref{app:delta-violation-empirical}.
  \emph{Right, sufficiency-(ii) violation:} the violation is specific to the treated group. At $T=0$ (left sub-panel) the structure matches
  what $\phi$ was trained on, so sufficiency-(ii) holds; at $T=1$ (right
  sub-panel) a violating $X\!\to\!S$ edge (dashed red) appears, encoding
  heterogeneity in the treatment effect on $S$ that depends on features
  of $X$ unrecoverable from untreated data. Analyzed in
  Appendix~\ref{app:suff-violation-empirical}.}
\label{fig:dags}
\end{figure}

\subsection{Empirical analysis of surrogacy violation (\texorpdfstring{$\delta$}{delta})}\label{app:delta-violation-empirical}

We complement the above discussion with an empirical sweep over the degree of \emph{surrogacy violation}, i.e.\ the extent to which $Y^*$ depends on treatment $T$ through paths that do not pass through the surrogate $S$. Building on the synthetic DGP of Section~\ref{sec:exp_synth}, we introduce a scalar $\delta\ge 0$ that controls the strength of the direct effect of treatment on $Y^*$ that does not pass through $S$:
\[
  Y^*_i(t) = h(S_i(t)) + \delta \cdot t \cdot g_\perp(Z^Y_i) + \varepsilon_i,
\]
where $Z^Y$ denotes latent directions that do not affect $S$, and $g_\perp$ is fixed across $\delta$. Since the direct term is multiplied by treatment, it is absent from the untreated historical data and cannot be learned by the historical outcome model. At $\delta=0$, surrogacy holds and the treatment effect on $Y^*$ is fully mediated by $S$. As $\delta$ increases, a larger share of the treatment effect bypasses $S$. We scale the direct and mediated components to have equal variance at $\delta=1$; at $\delta=2$, the direct component accounts for 80\% of the CATE variance.

We use the same three-fold DML learner as in Section~\ref{sec:experiments}. We run each method using both the historical outcome-model prediction $\widehat Y$ and the true experimental outcome $Y$ as the CATE target. We vary $\delta\in\{0,0.5,1,2\}$ and $n\in\{50,100,250,500,750\}$ and report mean $\pm$ one standard error over the same ten historical and experimental samples used in Section~\ref{sec:exp_synth}. Remaining implementation details are in Appendix~\ref{app:synth_exp_extra}.

\begin{figure}[t]
  \centering
  \begin{subfigure}[t]{\textwidth}\centering
    \includegraphics[width=\textwidth]{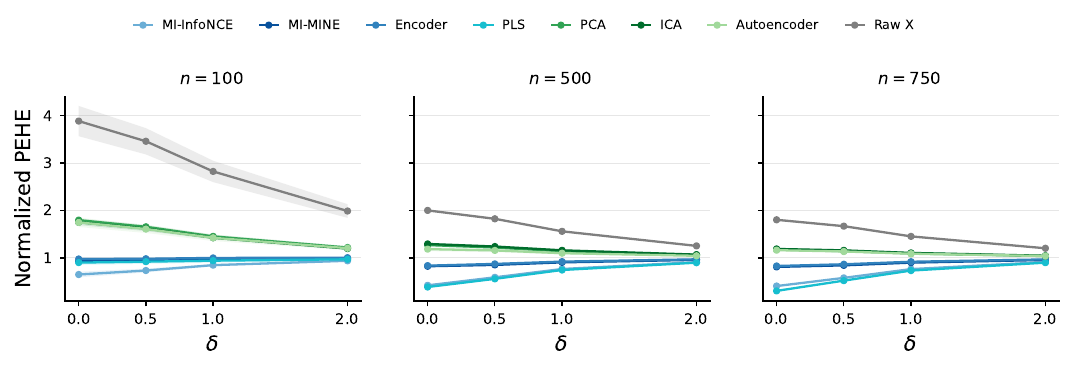}
    \caption{Normalized PEHE, $\hat Y$ target.}
  \end{subfigure}\\[2pt]
  \begin{subfigure}[t]{\textwidth}\centering
    \includegraphics[width=\textwidth]{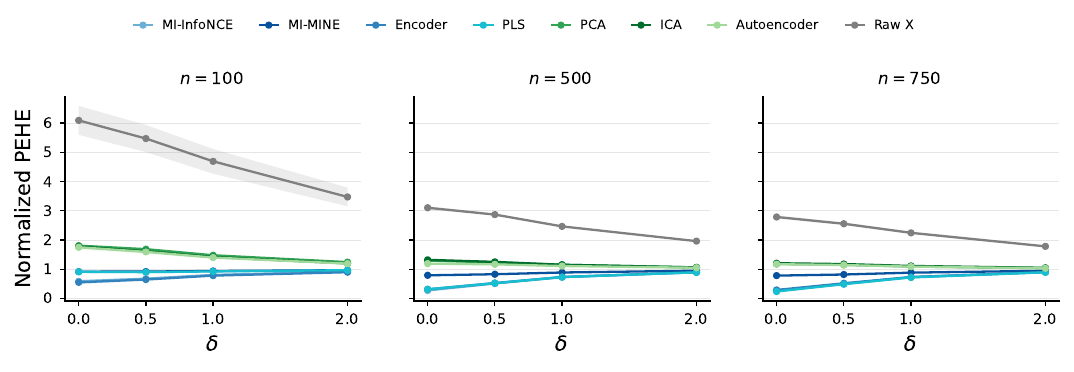}
    \caption{Normalized PEHE, $Y$ target.}
  \end{subfigure}
  \caption{\textbf{Surrogacy violation, DML -- normalized PEHE.}
    The outcome-aware representations lose accuracy as more of the treatment effect bypasses $S$, but remain more accurate than unsupervised compression and raw $X$ throughout the sweep.}
  \label{fig:delta-sweep-dml-pehe}
\end{figure}

\begin{figure}[t]
  \centering
  \begin{subfigure}[t]{\textwidth}\centering
    \includegraphics[width=\textwidth]{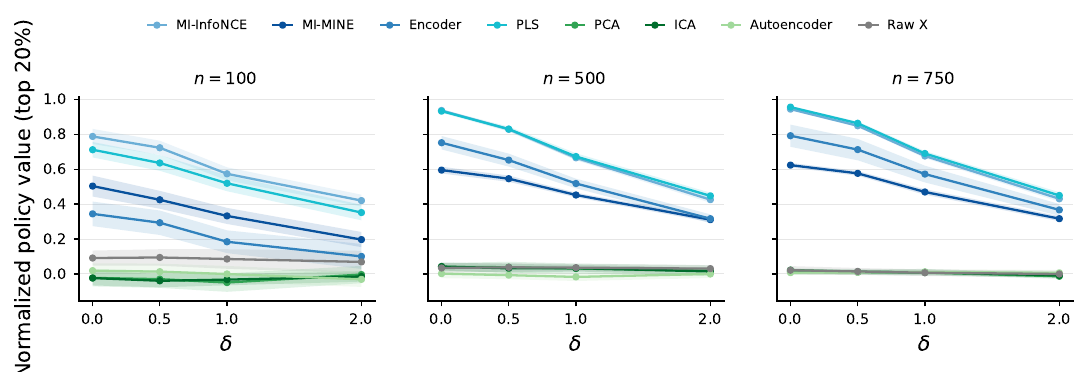}
    \caption{Normalized policy value (top 20\%), $\hat Y$ target.}
  \end{subfigure}\\[2pt]
  \begin{subfigure}[t]{\textwidth}\centering
    \includegraphics[width=\textwidth]{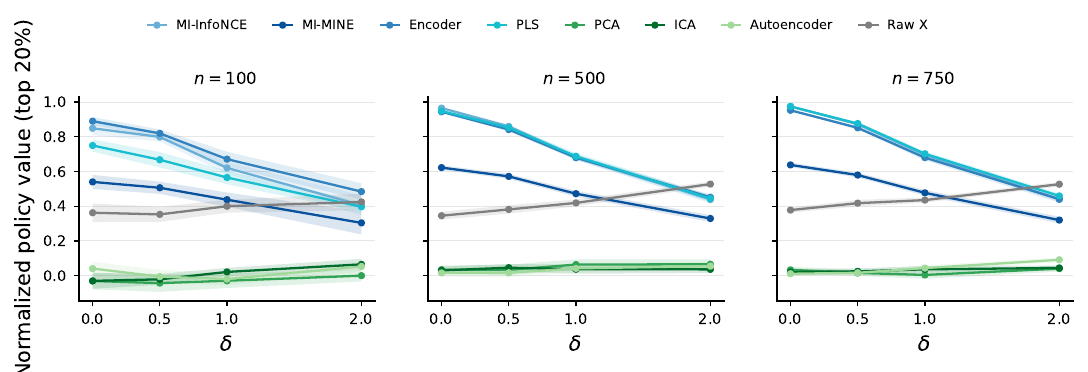}
    \caption{Normalized policy value (top 20\%), $Y$ target.}
  \end{subfigure}
  \caption{\textbf{Surrogacy violation, DML -- normalized top-20\%
    policy value.} The policy advantage declines as the direct effect grows, but remains positive at $\delta=2$ because the surrogate-mediated component remains informative.}
  \label{fig:delta-sweep-dml-policy}
\end{figure}

\paragraph{Results.}
Figures~\ref{fig:delta-sweep-dml-pehe} and~\ref{fig:delta-sweep-dml-policy} report normalized PEHE and policy value across $\delta$, respectively.
At $n=500$, increasing $\delta$ from zero to one raises MI-InfoNCE's normalized PEHE from $0.423$ to $0.767$, while its policy value goes from $0.936$ to $0.665$. At $\delta=2$, its PEHE reaches $0.911$ and its policy value falls to $0.425$; PLS follows the same pattern, and the trend is similar across sample sizes. When we use the true experimental outcome, raw $X$ increasingly recovers the direct-effect ranking, although its normalized PEHE at $\delta=2$ remains substantially higher than that of MI-InfoNCE ($1.961$ versus $0.898$). For raw $X$ and unsupervised methods, normalized PEHE declines with $\delta$ even though absolute PEHE increases because the standard deviation of the CATE also increases. Thus historical supervision remains useful under substantial surrogacy violations, but its advantage narrows as more of the treatment effect bypasses $S$.

\section{Prediction losses as mutual information lower bounds}\label{app:mi_bounds}
In addition to neural MI estimators (InfoNCE, MINE), we also report results using simple prediction losses (BCE for binary outcomes, MSE for continuous). Here we show that both BCE and MSE losses lower-bound the mutual information $I(X; Y)$, providing principled justification for their use.

\paragraph{Setup.} Since $I(X; Y) = H(Y) - H(Y \mid X)$ and $H(Y)$ is fixed, maximizing MI is equivalent to minimizing $H(Y \mid X)$. We show each loss upper-bounds $H(Y \mid X)$, so minimizing the loss tightens a lower bound on MI. Conditioning on additional covariates $Z$ follows identically.

\paragraph{BCE ($Y \mid X$ Bernoulli).} The expected BCE loss decomposes via $H(P, Q) = H(P) + D_{\mathrm{KL}}(P \parallel Q)$:
\[
\mathcal{L}_{\mathrm{BCE}}(\theta) = \mathbb{E}_X[H(P_{\mathrm{true}}(Y \mid X), P_\theta(Y \mid X))] = H(Y \mid X) + \mathbb{E}_X[D_{\mathrm{KL}}(P_{\mathrm{true}} \parallel P_\theta)].
\]
Since $D_{\mathrm{KL}} \geq 0$, $\mathcal{L}_{\mathrm{BCE}} \geq H(Y \mid X)$, hence $H(Y) - \mathcal{L}_{\mathrm{BCE}} \leq I(X; Y)$. Minimizing $\mathcal{L}_{\mathrm{BCE}}$ drives $D_{\mathrm{KL}} \to 0$ (given sufficient capacity), making the bound tight.

\paragraph{MSE ($Y \mid X$ real-valued, $\mu(X) = \mathbb{E}[Y \mid X]$).} The bias-variance decomposition gives:
\[
\mathcal{L}_{\mathrm{MSE}}(\theta) = \mathbb{E}_X[\mathrm{Var}(Y \mid X)] + \mathbb{E}_X[(\mu(X) - f_\theta(X))^2] \geq \mathbb{E}_X[\mathrm{Var}(Y \mid X)].
\]
To connect to MI, we bound the conditional differential entropy $h(Y \mid X)$. The Gaussian distribution has maximum differential entropy among distributions with a given variance \citep{shannon1948mathematical}, so pointwise: $h(Y \mid X = x) \leq \tfrac{1}{2}\ln(2\pi e \cdot \mathrm{Var}(Y \mid X = x))$. Taking $\mathbb{E}_X$ and applying Jensen's inequality:
\[
h(Y \mid X) \leq \tfrac{1}{2}\ln(2\pi e \cdot \mathbb{E}_X[\mathrm{Var}(Y \mid X)]) \leq \tfrac{1}{2}\ln(2\pi e \cdot \mathcal{L}_{\mathrm{MSE}}(\theta)).
\]
Hence $h(Y) - \tfrac{1}{2}\ln(2\pi e \cdot \mathcal{L}_{\mathrm{MSE}}) \leq I(X; Y)$, and minimizing $\mathcal{L}_{\mathrm{MSE}}$ tightens this bound.

\paragraph{Tightness.} The BCE bound is tight at convergence (with sufficient capacity). The MSE bound involves two additional gaps: the max-entropy gap (tight when $Y \mid X$ is Gaussian) and Jensen's gap (tight when $\mathrm{Var}(Y \mid X)$ is constant in $X$). Both close under additive Gaussian noise $Y = \mu(X) + \varepsilon$, $\varepsilon \sim \mathcal{N}(0, \sigma^2)$. In general the MSE bound is valid but looser, which is consistent with our empirical finding (Section~\ref{sec:experiments}) that MINE and InfoNCE achieve modestly better performance than the prediction-loss variant.

\paragraph{Application to the combined objective.} The bounds above apply to each term in Equation~\ref{eq:objective_function}: setting $(X, Y \mid Z) = (\phi(X), Y^* \mid S)$ yields the outcome-relevant MI bound, and $(X, Y) = (\phi(X), S)$ yields the surrogate-relevant bound.

\section{Empirical results (cont.)}\label{app:empirical_extra}
\paragraph{Libraries used.}
All experiment code is written in Python. We implement the neural-network encoder architecture using PyTorch~\citep{paszke2019pytorch}. For baseline machine-learning models, preprocessing, and evaluation metrics we use scikit-learn~\citep{scikit-learn}; the X-learner is implemented via CausalML~\citep{chen2020causalml}, and the DML estimator via EconML~\citep{econml}. Synthetic outcome and propensity nuisance models use scikit-learn's \texttt{HistGradientBoostingRegressor}; dataset-specific settings are given below.

\paragraph{Our method: encoder backbone.}
For the synthetic experiments, all neural methods use the same feedforward encoder with ReLU activations, $X \in \mathbb{R}^{d} \to 128 \to 64 \to 32 \to m$, where $m=10$ is the bottleneck dimension. We train with Adam (lr $10^{-3}$, batch~$256$, up to $200$ epochs) and early stopping on training loss with patience~$15$ and stopping criterion of $10^{-4}$. The semi-synthetic experiment uses the smaller backbone and validation-based early stopping described in Appendix~\ref{app:semi_synthetic}, with $m=10$ in the main analysis and $m=5$ as a sensitivity analysis.

\paragraph{Our method: training variants.}
We instantiate the objective in Equation~\ref{eq:objective_function} in three ways, sharing the encoder backbone but using different heads.
\begin{itemize}
\item \texttt{encoder\_pred}: prediction-based (MSE bound on MI). A surrogate head $\phi(X) \to \hat S$ and an outcome head $(\phi(X), S) \to \hat Y$, both with one $64$-wide hidden layer and ReLU. Loss $\mathcal{L} = \mathrm{MSE}(\hat Y, Y) + \lambda_S \cdot \mathrm{MSE}(\hat S, S)$.
\item \texttt{mi\_infonce\_cond}: InfoNCE on the joint $(Y, S)$ anchor. Projection heads $g_\phi: \phi(X) \to \mathbb{R}^{p}$ and $g_{YS}: (Y, S) \to \mathbb{R}^{p}$ feed a symmetric in-batch contrastive loss.
\item \texttt{mi\_mine\_cond}: Donsker--Varadhan estimator on the joint
  $(Y, S)$ anchor: a critic network $T_{(Y,S)}(\phi(X), (Y, S)) \to \mathbb{R}$ scores joint vs.\ shuffled pairs.
\end{itemize}
Both MI variants additionally include a marginal $S$-anchor term weighted by $(\lambda_S - 1)$, recovering the chain-rule decomposition of Section~\ref{sec:objective_function}. We use $\lambda_S = 2$ throughout the reported experiments.

\subsection{Synthetic dataset: details}\label{app:synth_exp_extra}
\subsubsection{Description}
\paragraph{Data-generating process (full).}
We construct the ten-dimensional latent variable $Z=X\,W_{XZ}^{\!\top}$ from linear combinations of a small subset of the covariates. We split $Z=(Z_s,Z_y)$, where $Z_s,Z_y\in\mathbb{R}^5$ depend on two disjoint sets of ten covariates. We draw the coefficients defining these relationships once and hold them fixed across all data samples. In the main synthetic experiment in Section~\ref{sec:exp_synth}, we set $\alpha=\delta=0$, so only $Z_s$ affects the surrogates, outcome, and treatment effect. We use $Z_y$ only in the sufficiency- and surrogacy-violation experiments.

The surrogate $S \in \mathbb{R}^{10}$ is a noisy linear function of $Z_s$, shifted by a $Z_s$-driven treatment effect:
\[
S_i(t) = \beta_0 + Z_{s,i}\beta_1^{\!\top} + t \cdot \tau_S(Z_{s,i}) + \epsilon_{S,i}, \qquad \epsilon_{S,i} \sim \mathcal{N}(\mathbf{0}, \sigma_S^2 \mathbf{I}),\]
with $\tau_S(z) = \gamma_0 + z\gamma_z^{\!\top}$, $\beta_0 \sim \mathcal{N}(0, 0.5^2)$, $\beta_1 \sim \mathcal{N}(0, 1)$, $\gamma_0 \sim \mathcal{N}(0, 0.5^2)$, $\gamma_z \sim \mathcal{N}(0, 0.3^2)$, and $\sigma_S = 0.1$. The primary outcome depends on $X$ only through $S$:
$$Y_i = S_i w_{SY} + \epsilon_{Y,i}, \qquad \epsilon_{Y,i} \sim \mathcal{N}(0, \sigma_Y^2),$$
with $w_{SY} \sim \mathcal{N}(\mathbf{0}, \mathbf{I})$ and $\sigma_Y = 0.1$.
Both the baseline value of $S$ and the treatment effect on $S$ depend only on $Z_s$ in the main experiment. For the sufficiency-violation experiment, we also allow $Z_y$ to affect the treatment-induced change in $S$. We scale the $Z_s$ and $Z_y$ components so that they contribute equal variance to the scalar CATE; the weights $\sqrt{1-\alpha^2}$ and $\alpha$ therefore keep its variance fixed as $\alpha$ varies. For the surrogacy-violation experiment, we add a direct treatment effect on $Y$ that depends on $Z_y$ and scale it to contribute the same CATE variance as the mediated component at $\delta=1$. The $w_{SY}$ weights mapping $S \to Y$ are $\mathcal{N}(\mathbf{0}, \mathbf{I})$; observation noise is $\sigma_S = \sigma_Y = 0.1$ on both surrogate and outcome.

\paragraph{Methods, training, and evaluation.}
We compare our three representations (\texttt{encoder\_pred} -- MSE prediction heads on $S$ and $Y$, \texttt{mi\_mine\_cond} -- MINE on the joint $(Y,S)$ anchor, and \texttt{mi\_infonce\_cond} -- InfoNCE on the joint $(Y,S)$ anchor) against \texttt{pls} (supervised linear), \texttt{autoencoder}, \texttt{pca}, and \texttt{ica} (unsupervised), and two raw-feature comparisons with $\phi(X)\equiv X$, whose outcome models fit either of $h(X)$ or $h(X,S)$. Neural encoders use the MLP backbone $1000 \to 128 \to 64 \to 32 \to \phi$ with ReLU activations, trained with Adam (lr $10^{-3}$, batch 256, up to $200$ epochs, early stopping on training loss with patience $15$ and min-delta $10^{-4}$); the surrogate-prediction weight is $\lambda_S = 2$, and the InfoNCE temperature is $0.07$. Every reported synthetic experiment uses a 50-tree random forest for the outcome prediction model trained on historical data and \texttt{HistGradientBoostingRegressor} with $40$ iterations inside the CATE learners. We report both X-learner and DML using either the predicted outcome $\hat Y$ or the observed outcome $Y$ as the CATE target.

\paragraph{Metrics, seeds, and compute.}
We report normalized PEHE $\sqrt{\mathbb{E}[(\hat{\tau} - \tau)^2]} / \mathrm{std}(\tau)$ (scale-invariant across DGPs) and normalized top-$k$ policy value $(V_{\pi} - V_{\mathrm{random}}) / (V_{\mathrm{oracle}} - V_{\mathrm{random}})$, where $V_{\pi}$ is the mean of $\tau_{\mathrm{true}}$ over the top-$k$ units ranked by $\hat{\tau}$, $V_{\mathrm{random}}$ is the population mean, and $V_{\mathrm{oracle}}$ is $V_{\pi}$ with $\hat{\tau}$ replaced by
$\tau_{\mathrm{true}}$. By construction $0$ corresponds to a random ranker and $1$ to the oracle policy. All synthetic experiments fix one DGP ($\texttt{dgp\_seed}=42$) and average over the same 10 historical and experimental resamples. The main synthetic and violation experiments use $n\in\{50,100,250,500,750\}$; error bars show $\pm 1$ standard error.

\subsubsection{Additional results: figures}\label{app:synth-figures}
\Cref{fig:synth_appendix_xl} reports the X-learner results across both CATE targets ($\hat Y$ and $Y$); \Cref{fig:synth_appendix_dml} reports the corresponding DML results. Every panel uses the same DGP, sample-size grid, methods, and 10 historical and experimental resamples as the main analysis.

\begin{figure}[h]
    \centering
    \begin{subfigure}[t]{0.48\linewidth}
        \centering
        \includegraphics[width=\linewidth]{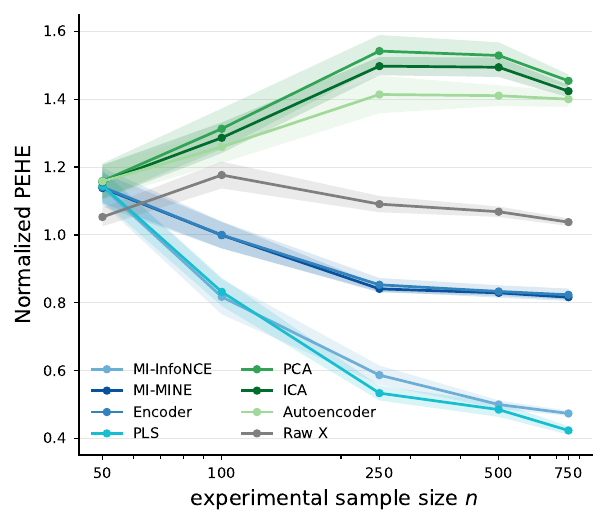}
        \caption{PEHE, $\hat{Y}$ target.}
        \label{fig:synth_appendix_xl_pehe_pred}
    \end{subfigure}
    \hfill
    \begin{subfigure}[t]{0.48\linewidth}
        \centering
        \includegraphics[width=\linewidth]{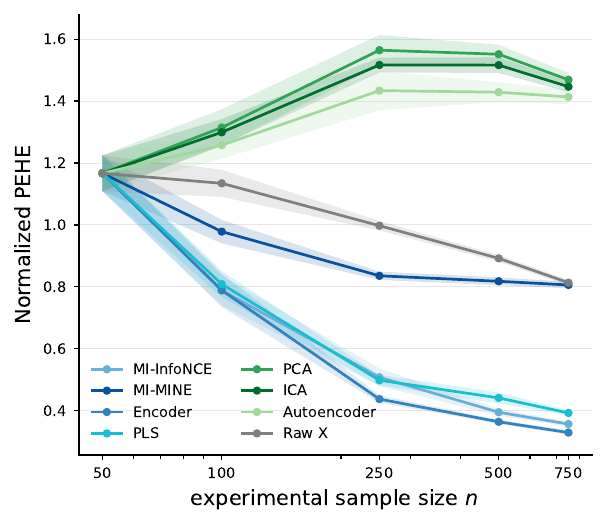}
        \caption{PEHE, $Y$ target.}
        \label{fig:synth_appendix_xl_pehe_true}
    \end{subfigure}

    \vspace{0.4em}

    \begin{subfigure}[t]{0.48\linewidth}
        \centering
        \includegraphics[width=\linewidth]{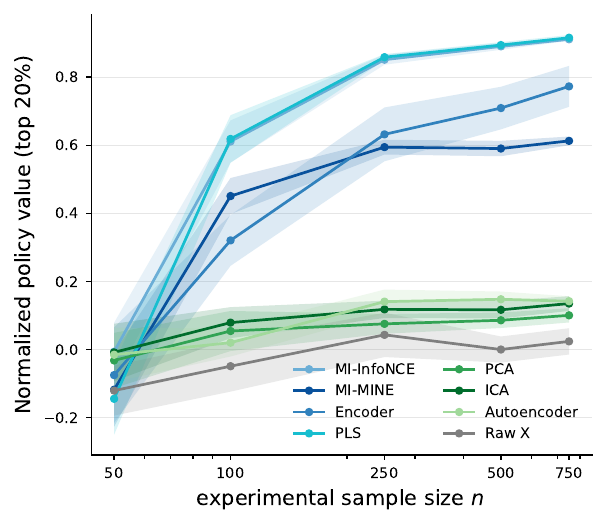}
        \caption{Policy@20\%, $\hat{Y}$ target.}
        \label{fig:synth_appendix_xl_policy_pred}
    \end{subfigure}
    \hfill
    \begin{subfigure}[t]{0.48\linewidth}
        \centering
        \includegraphics[width=\linewidth]{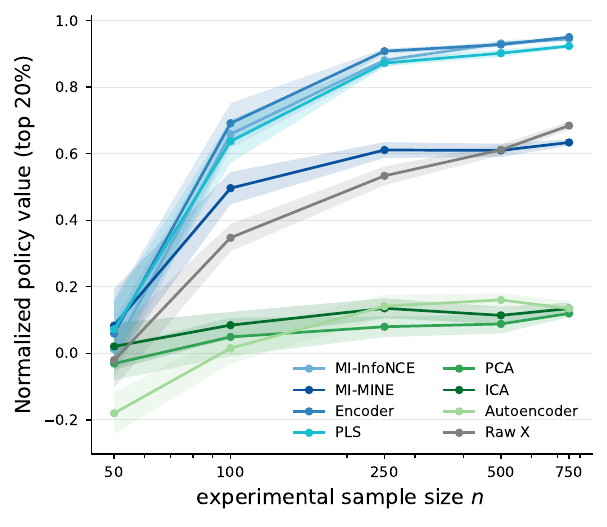}
        \caption{Policy@20\%, $Y$ target.}
        \label{fig:synth_appendix_xl_policy_true}
    \end{subfigure}

    \caption{\textbf{Synthetic data} results using the X-learner. Rows: error metric (PEHE; normalized top-20\% policy value). Columns: CATE target ($\hat{Y}$ vs.\ $Y$). Curves show means over 10 historical and experimental resamples from the same DGP; shaded regions show $\pm 1$ standard error.}
    \label{fig:synth_appendix_xl}
\end{figure}

\begin{figure}[h]
    \centering
    \begin{subfigure}[t]{0.48\linewidth}
        \centering
        \includegraphics[width=\linewidth]{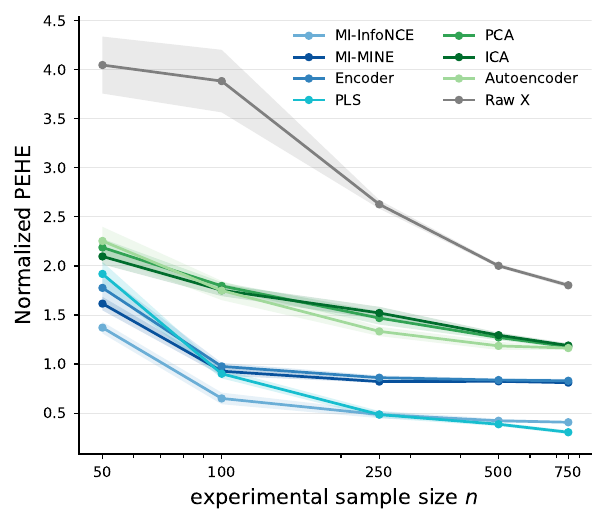}
        \caption{PEHE, $\hat{Y}$ target.}
        \label{fig:synth_appendix_dml_pehe_pred}
    \end{subfigure}
    \hfill
    \begin{subfigure}[t]{0.48\linewidth}
        \centering
        \includegraphics[width=\linewidth]{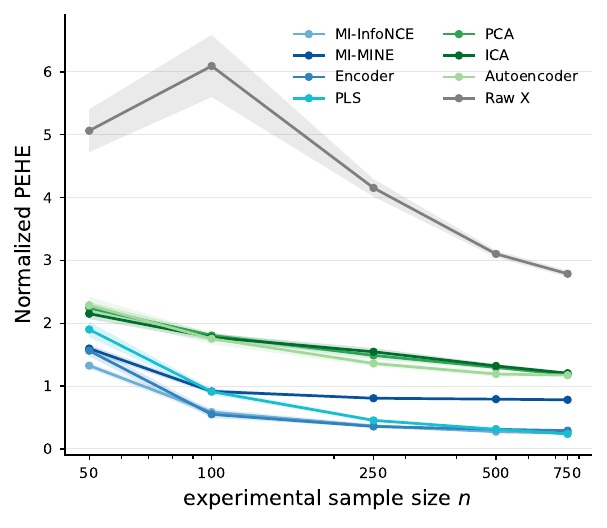}
        \caption{PEHE, $Y$ target.}
        \label{fig:synth_appendix_dml_pehe_true}
    \end{subfigure}

    \vspace{0.4em}

    \begin{subfigure}[t]{0.48\linewidth}
        \centering
        \includegraphics[width=\linewidth]{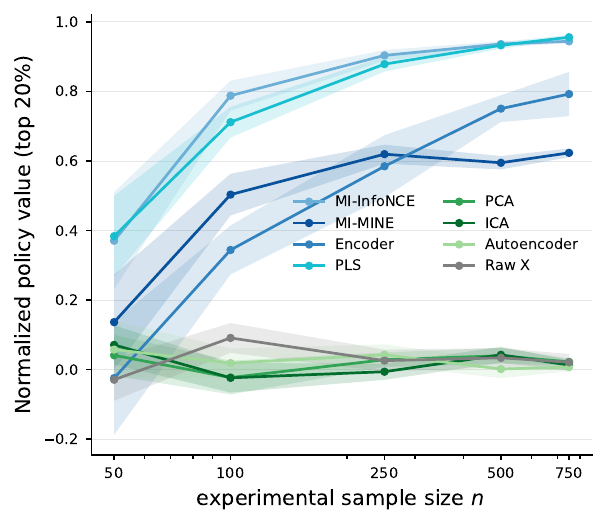}
        \caption{Policy@20\%, $\hat{Y}$ target.}
        \label{fig:synth_appendix_dml_policy_pred}
    \end{subfigure}
    \hfill
    \begin{subfigure}[t]{0.48\linewidth}
        \centering
        \includegraphics[width=\linewidth]{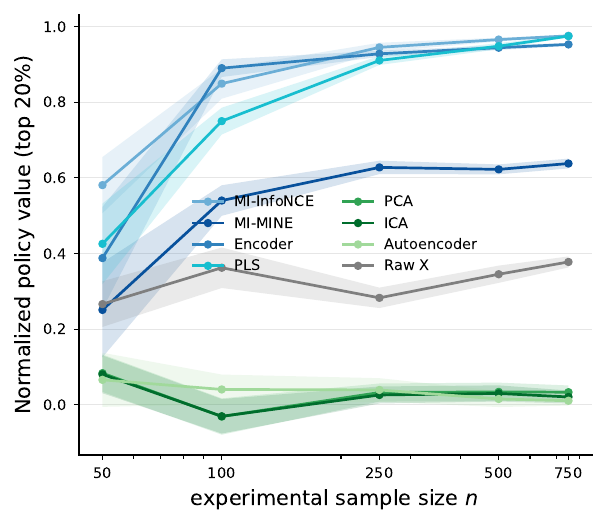}
        \caption{Policy@20\%, $Y$ target.}
        \label{fig:synth_appendix_dml_policy_true}
    \end{subfigure}

    \caption{\textbf{Synthetic data} results using the DML learner using predicted and true outcomes as CATE targets. Curves show means over the same 10 historical and experimental resamples as the main analysis; shaded regions show $\pm 1$ standard error.}
    \label{fig:synth_appendix_dml}
\end{figure}

\subsubsection{Additional results: tabulated PEHE and top-20\% policy value}\label{app:synth-tables}

Tables~\ref{tab:synth_pehe_norm_xl}--\ref{tab:synth_policy_norm_20_dml} report X-learner and DML results for the nine methods in the main comparison. For the true-$Y$ target, the two raw-$X$ variants coincide because no historical outcome model is used. All entries use the same 10 paired historical and experimental resamples.

\begin{table}[t]
\centering
\resizebox{\linewidth}{!}{%
\begin{tabular}{l|ccccc|ccccc}
\toprule
 & \multicolumn{5}{c}{$\widehat Y$ target} & \multicolumn{5}{c}{$Y$ target} \\
Method & $n=50$ & $n=100$ & $n=250$ & $n=500$ & $n=750$ & $n=50$ & $n=100$ & $n=250$ & $n=500$ & $n=750$ \\
\midrule
MI-InfoNCE & 1.141$\pm$0.048 & \textbf{0.818$\pm$0.052} & 0.587$\pm$0.028 & 0.500$\pm$0.012 & 0.474$\pm$0.008 & \textbf{1.167$\pm$0.058} & \textbf{0.789$\pm$0.048} & 0.508$\pm$0.028 & 0.395$\pm$0.012 & 0.357$\pm$0.008 \\
MI-MINE & 1.138$\pm$0.044 & 0.999$\pm$0.039 & 0.841$\pm$0.010 & 0.829$\pm$0.009 & 0.816$\pm$0.007 & \textbf{1.167$\pm$0.058} & 0.978$\pm$0.038 & 0.836$\pm$0.013 & 0.818$\pm$0.013 & 0.806$\pm$0.010 \\
Encoder & 1.141$\pm$0.057 & 1.000$\pm$0.040 & 0.853$\pm$0.020 & 0.833$\pm$0.018 & 0.824$\pm$0.018 & \textbf{1.167$\pm$0.058} & 0.789$\pm$0.054 & \textbf{0.437$\pm$0.012} & \textbf{0.364$\pm$0.006} & \textbf{0.329$\pm$0.007} \\
PLS & 1.145$\pm$0.046 & 0.832$\pm$0.041 & \textbf{0.534$\pm$0.022} & \textbf{0.485$\pm$0.021} & \textbf{0.423$\pm$0.012} & \textbf{1.167$\pm$0.058} & 0.809$\pm$0.042 & 0.498$\pm$0.017 & 0.441$\pm$0.018 & 0.393$\pm$0.008 \\
PCA & 1.159$\pm$0.052 & 1.313$\pm$0.060 & 1.542$\pm$0.047 & 1.529$\pm$0.039 & 1.454$\pm$0.019 & \textbf{1.167$\pm$0.058} & 1.314$\pm$0.059 & 1.564$\pm$0.049 & 1.551$\pm$0.031 & 1.469$\pm$0.022 \\
ICA & 1.156$\pm$0.050 & 1.286$\pm$0.045 & 1.498$\pm$0.027 & 1.494$\pm$0.029 & 1.424$\pm$0.018 & \textbf{1.167$\pm$0.058} & 1.299$\pm$0.043 & 1.517$\pm$0.024 & 1.516$\pm$0.025 & 1.447$\pm$0.018 \\
Autoencoder & 1.158$\pm$0.051 & 1.260$\pm$0.047 & 1.414$\pm$0.055 & 1.411$\pm$0.030 & 1.400$\pm$0.023 & \textbf{1.167$\pm$0.058} & 1.257$\pm$0.044 & 1.434$\pm$0.063 & 1.429$\pm$0.032 & 1.413$\pm$0.023 \\
Raw $X$, $h(X)$ & \textbf{1.053$\pm$0.028} & 1.177$\pm$0.040 & 1.091$\pm$0.024 & 1.068$\pm$0.015 & 1.038$\pm$0.011 & \textbf{1.167$\pm$0.058} & 1.134$\pm$0.044 & 0.997$\pm$0.015 & 0.892$\pm$0.011 & 0.813$\pm$0.010 \\
Raw $X$, $h(X,S)$ & 1.149$\pm$0.047 & 1.109$\pm$0.041 & 0.976$\pm$0.022 & 0.861$\pm$0.010 & 0.816$\pm$0.011 & --- & --- & --- & --- & --- \\
\bottomrule
\end{tabular}}
\caption{Normalized PEHE on the synthetic data, X-learner ($\dim(\phi)=10$). Entries are means $\pm$ standard errors over 10 paired data resamples; lower is better. For the $Y$ target, the two raw-$X$ variants coincide, so we report a single value under $h(X)$.}
\label{tab:synth_pehe_norm_xl}
\end{table}

\begin{table}[t]
\centering
\resizebox{\linewidth}{!}{%
\begin{tabular}{l|ccccc|ccccc}
\toprule
 & \multicolumn{5}{c}{$\widehat Y$ target} & \multicolumn{5}{c}{$Y$ target} \\
Method & $n=50$ & $n=100$ & $n=250$ & $n=500$ & $n=750$ & $n=50$ & $n=100$ & $n=250$ & $n=500$ & $n=750$ \\
\midrule
MI-InfoNCE & \textbf{-0.006$\pm$0.085} & 0.611$\pm$0.062 & 0.851$\pm$0.015 & 0.890$\pm$0.010 & 0.911$\pm$0.005 & 0.012$\pm$0.084 & 0.659$\pm$0.042 & 0.881$\pm$0.020 & \textbf{0.931$\pm$0.010} & 0.944$\pm$0.005 \\
MI-MINE & -0.117$\pm$0.109 & 0.451$\pm$0.053 & 0.594$\pm$0.022 & 0.590$\pm$0.023 & 0.613$\pm$0.013 & \textbf{0.084$\pm$0.111} & 0.496$\pm$0.048 & 0.611$\pm$0.024 & 0.610$\pm$0.020 & 0.634$\pm$0.010 \\
Encoder & -0.075$\pm$0.140 & 0.321$\pm$0.075 & 0.632$\pm$0.079 & 0.709$\pm$0.063 & 0.773$\pm$0.060 & 0.060$\pm$0.102 & \textbf{0.691$\pm$0.060} & \textbf{0.908$\pm$0.008} & 0.928$\pm$0.005 & \textbf{0.950$\pm$0.003} \\
PLS & -0.144$\pm$0.105 & \textbf{0.618$\pm$0.070} & \textbf{0.858$\pm$0.009} & \textbf{0.894$\pm$0.009} & \textbf{0.915$\pm$0.008} & 0.072$\pm$0.084 & 0.637$\pm$0.066 & 0.872$\pm$0.011 & 0.902$\pm$0.012 & 0.923$\pm$0.006 \\
PCA & -0.032$\pm$0.081 & 0.055$\pm$0.059 & 0.076$\pm$0.029 & 0.086$\pm$0.023 & 0.101$\pm$0.021 & -0.031$\pm$0.050 & 0.049$\pm$0.057 & 0.080$\pm$0.031 & 0.088$\pm$0.029 & 0.120$\pm$0.011 \\
ICA & -0.008$\pm$0.084 & 0.079$\pm$0.045 & 0.118$\pm$0.025 & 0.117$\pm$0.025 & 0.135$\pm$0.022 & 0.021$\pm$0.070 & 0.085$\pm$0.040 & 0.135$\pm$0.030 & 0.114$\pm$0.026 & 0.135$\pm$0.018 \\
Autoencoder & -0.015$\pm$0.080 & 0.020$\pm$0.041 & 0.141$\pm$0.035 & 0.148$\pm$0.023 & 0.142$\pm$0.015 & -0.180$\pm$0.063 & 0.015$\pm$0.046 & 0.142$\pm$0.034 & 0.161$\pm$0.020 & 0.134$\pm$0.023 \\
Raw $X$, $h(X)$ & -0.120$\pm$0.073 & -0.049$\pm$0.075 & 0.043$\pm$0.065 & 0.000$\pm$0.038 & 0.024$\pm$0.039 & -0.020$\pm$0.082 & 0.348$\pm$0.041 & 0.533$\pm$0.028 & 0.611$\pm$0.016 & 0.684$\pm$0.012 \\
Raw $X$, $h(X,S)$ & -0.131$\pm$0.066 & 0.354$\pm$0.040 & 0.519$\pm$0.037 & 0.617$\pm$0.014 & 0.668$\pm$0.013 & --- & --- & --- & --- & --- \\
\bottomrule
\end{tabular}}
\caption{Normalized top-20\% policy value on the synthetic data, X-learner ($\dim(\phi)=10$). Entries are means $\pm$ standard errors over 10 paired data resamples; higher is better. For the $Y$ target, the two raw-$X$ variants coincide, so we report a single value under $h(X)$.}
\label{tab:synth_policy_norm_20_xl}
\end{table}

\begin{table}[t]
\centering
\resizebox{\linewidth}{!}{%
\begin{tabular}{l|ccccc|ccccc}
\toprule
 & \multicolumn{5}{c}{$\widehat Y$ target} & \multicolumn{5}{c}{$Y$ target} \\
Method & $n=50$ & $n=100$ & $n=250$ & $n=500$ & $n=750$ & $n=50$ & $n=100$ & $n=250$ & $n=500$ & $n=750$ \\
\midrule
MI-InfoNCE & \textbf{1.371$\pm$0.059} & \textbf{0.650$\pm$0.059} & 0.486$\pm$0.037 & 0.423$\pm$0.019 & 0.408$\pm$0.011 & \textbf{1.326$\pm$0.053} & 0.586$\pm$0.056 & 0.364$\pm$0.032 & \textbf{0.274$\pm$0.012} & 0.252$\pm$0.010 \\
MI-MINE & 1.616$\pm$0.069 & 0.928$\pm$0.046 & 0.822$\pm$0.013 & 0.825$\pm$0.008 & 0.811$\pm$0.008 & 1.598$\pm$0.059 & 0.916$\pm$0.036 & 0.806$\pm$0.008 & 0.791$\pm$0.008 & 0.781$\pm$0.007 \\
Encoder & 1.776$\pm$0.151 & 0.977$\pm$0.036 & 0.861$\pm$0.023 & 0.837$\pm$0.017 & 0.830$\pm$0.019 & 1.563$\pm$0.091 & \textbf{0.550$\pm$0.060} & \textbf{0.357$\pm$0.013} & 0.312$\pm$0.008 & 0.291$\pm$0.005 \\
PLS & 1.919$\pm$0.123 & 0.902$\pm$0.047 & \textbf{0.486$\pm$0.036} & \textbf{0.387$\pm$0.019} & \textbf{0.306$\pm$0.020} & 1.901$\pm$0.120 & 0.911$\pm$0.048 & 0.454$\pm$0.028 & 0.313$\pm$0.014 & \textbf{0.238$\pm$0.011} \\
PCA & 2.188$\pm$0.089 & 1.796$\pm$0.037 & 1.469$\pm$0.067 & 1.271$\pm$0.028 & 1.177$\pm$0.016 & 2.243$\pm$0.094 & 1.806$\pm$0.042 & 1.487$\pm$0.063 & 1.296$\pm$0.034 & 1.180$\pm$0.015 \\
ICA & 2.097$\pm$0.081 & 1.752$\pm$0.060 & 1.521$\pm$0.063 & 1.295$\pm$0.027 & 1.189$\pm$0.023 & 2.151$\pm$0.086 & 1.776$\pm$0.057 & 1.546$\pm$0.059 & 1.321$\pm$0.028 & 1.205$\pm$0.023 \\
Autoencoder & 2.254$\pm$0.144 & 1.750$\pm$0.098 & 1.333$\pm$0.048 & 1.185$\pm$0.042 & 1.163$\pm$0.020 & 2.285$\pm$0.137 & 1.753$\pm$0.089 & 1.360$\pm$0.045 & 1.191$\pm$0.044 & 1.173$\pm$0.018 \\
Raw $X$, $h(X)$ & 4.047$\pm$0.290 & 3.883$\pm$0.319 & 2.628$\pm$0.043 & 2.001$\pm$0.019 & 1.803$\pm$0.022 & 5.063$\pm$0.340 & 6.094$\pm$0.494 & 4.153$\pm$0.139 & 3.103$\pm$0.055 & 2.787$\pm$0.051 \\
Raw $X$, $h(X,S)$ & 4.876$\pm$0.329 & 5.886$\pm$0.471 & 4.076$\pm$0.125 & 3.038$\pm$0.045 & 2.720$\pm$0.046 & --- & --- & --- & --- & --- \\
\bottomrule
\end{tabular}}
\caption{Normalized PEHE on the synthetic data, DML ($\dim(\phi)=10$). Entries are means $\pm$ standard errors over 10 paired data resamples; lower is better. For the $Y$ target, the two raw-$X$ variants coincide, so we report a single value under $h(X)$.}
\label{tab:synth_pehe_norm_dml}
\end{table}

\begin{table}[t]
\centering
\resizebox{\linewidth}{!}{%
\begin{tabular}{l|ccccc|ccccc}
\toprule
 & \multicolumn{5}{c}{$\widehat Y$ target} & \multicolumn{5}{c}{$Y$ target} \\
Method & $n=50$ & $n=100$ & $n=250$ & $n=500$ & $n=750$ & $n=50$ & $n=100$ & $n=250$ & $n=500$ & $n=750$ \\
\midrule
MI-InfoNCE & 0.371$\pm$0.140 & \textbf{0.788$\pm$0.043} & \textbf{0.904$\pm$0.015} & \textbf{0.936$\pm$0.008} & 0.944$\pm$0.006 & \textbf{0.581$\pm$0.074} & 0.849$\pm$0.040 & \textbf{0.945$\pm$0.012} & \textbf{0.966$\pm$0.005} & \textbf{0.976$\pm$0.002} \\
MI-MINE & 0.136$\pm$0.136 & 0.503$\pm$0.060 & 0.620$\pm$0.027 & 0.595$\pm$0.019 & 0.623$\pm$0.013 & 0.251$\pm$0.125 & 0.540$\pm$0.040 & 0.628$\pm$0.018 & 0.622$\pm$0.013 & 0.638$\pm$0.013 \\
Encoder & -0.025$\pm$0.163 & 0.344$\pm$0.070 & 0.585$\pm$0.090 & 0.750$\pm$0.039 & 0.793$\pm$0.064 & 0.388$\pm$0.135 & \textbf{0.890$\pm$0.023} & 0.928$\pm$0.008 & 0.944$\pm$0.006 & 0.953$\pm$0.004 \\
PLS & \textbf{0.384$\pm$0.116} & 0.712$\pm$0.044 & 0.879$\pm$0.021 & 0.933$\pm$0.009 & \textbf{0.956$\pm$0.007} & 0.426$\pm$0.107 & 0.750$\pm$0.035 & 0.910$\pm$0.012 & 0.948$\pm$0.008 & 0.975$\pm$0.002 \\
PCA & 0.041$\pm$0.059 & -0.023$\pm$0.049 & 0.028$\pm$0.021 & 0.040$\pm$0.023 & 0.022$\pm$0.010 & 0.084$\pm$0.049 & -0.031$\pm$0.048 & 0.032$\pm$0.024 & 0.034$\pm$0.024 & 0.033$\pm$0.018 \\
ICA & 0.071$\pm$0.056 & -0.024$\pm$0.042 & -0.006$\pm$0.023 & 0.042$\pm$0.023 & 0.013$\pm$0.015 & 0.080$\pm$0.050 & -0.030$\pm$0.046 & 0.026$\pm$0.023 & 0.030$\pm$0.021 & 0.021$\pm$0.017 \\
Autoencoder & 0.058$\pm$0.083 & 0.019$\pm$0.042 & 0.043$\pm$0.030 & 0.002$\pm$0.026 & 0.007$\pm$0.012 & 0.066$\pm$0.072 & 0.041$\pm$0.039 & 0.039$\pm$0.031 & 0.015$\pm$0.021 & 0.011$\pm$0.013 \\
Raw $X$, $h(X)$ & -0.029$\pm$0.061 & 0.091$\pm$0.043 & 0.026$\pm$0.032 & 0.034$\pm$0.028 & 0.021$\pm$0.024 & 0.266$\pm$0.060 & 0.362$\pm$0.053 & 0.283$\pm$0.027 & 0.345$\pm$0.023 & 0.378$\pm$0.015 \\
Raw $X$, $h(X,S)$ & 0.248$\pm$0.053 & 0.317$\pm$0.047 & 0.291$\pm$0.034 & 0.358$\pm$0.021 & 0.387$\pm$0.013 & --- & --- & --- & --- & --- \\
\bottomrule
\end{tabular}}
\caption{Normalized top-20\% policy value on the synthetic data, DML ($\dim(\phi)=10$). Entries are means $\pm$ standard errors over 10 paired data resamples; higher is better. For the $Y$ target, the two raw-$X$ variants coincide, so we report a single value under $h(X)$.}
\label{tab:synth_policy_norm_20_dml}
\end{table}

\subsection{Empirical analysis of representation dimension (\texorpdfstring{$m$}{m})}\label{app:m-sweep-empirical}

We probe the sensitivity of the estimators to the representation dimension
$m=\dim(\phi(X))$ at fixed latent dimensionality $z_{\mathrm{dim}}=10$. By
construction, only the five coordinates in $Z_s$ affect the surrogates and
outcome, so $m=5$ is the smallest dimension that can support exact
sufficiency. We sweep $m\in\{2,3,5,10,20,50\}$ at the canonical setting
$\alpha=\delta=0$ with $n\in\{50,100,250,500,750\}$, using three-fold DML and
both CATE targets; results are means $\pm$ one standard error over the same
10 historical and experimental resamples as the main synthetic
experiment. Raw $X$ does not depend on $m$ and provides a constant
reference.

\begin{figure}[t]
  \centering
  \begin{subfigure}[t]{\textwidth}\centering
    \includegraphics[width=\textwidth]{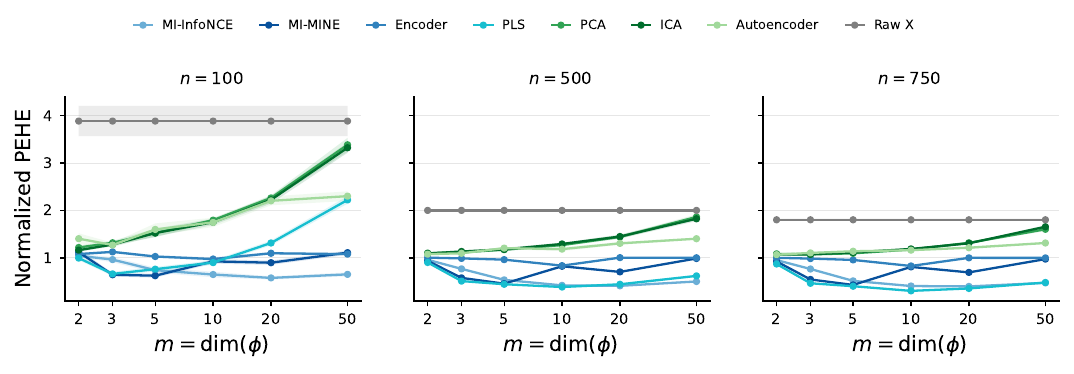}
    \caption{Normalized PEHE, $\hat Y$ target.}
  \end{subfigure}\\[2pt]
  \begin{subfigure}[t]{\textwidth}\centering
    \includegraphics[width=\textwidth]{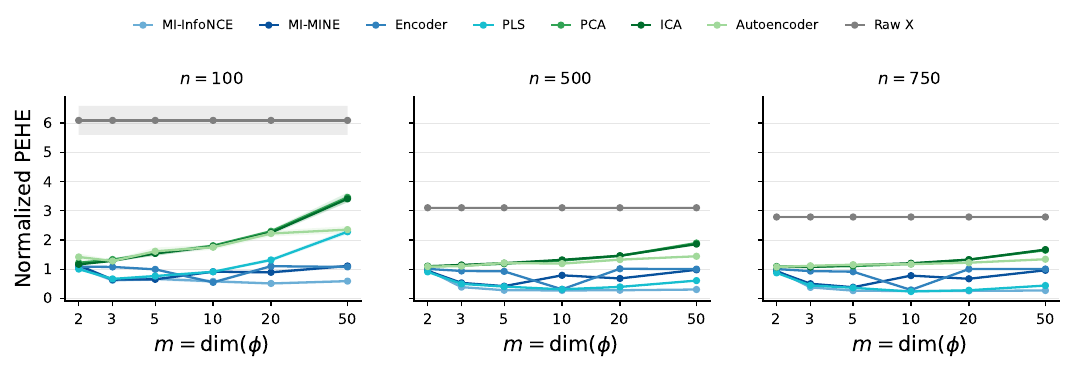}
    \caption{Normalized PEHE, $Y$ target.}
  \end{subfigure}
  \caption{\textbf{Representation dimension sweep, DML --- normalized
    PEHE.} Increasing $m$ through the outcome-relevant latent dimension substantially improves the best supervised methods. Beyond that point, additional dimensions provide no uniform benefit and can increase error for some methods.}
  \label{fig:m-sweep-dml-pehe}
\end{figure}

\begin{figure}[t]
  \centering
  \begin{subfigure}[t]{\textwidth}\centering
    \includegraphics[width=\textwidth]{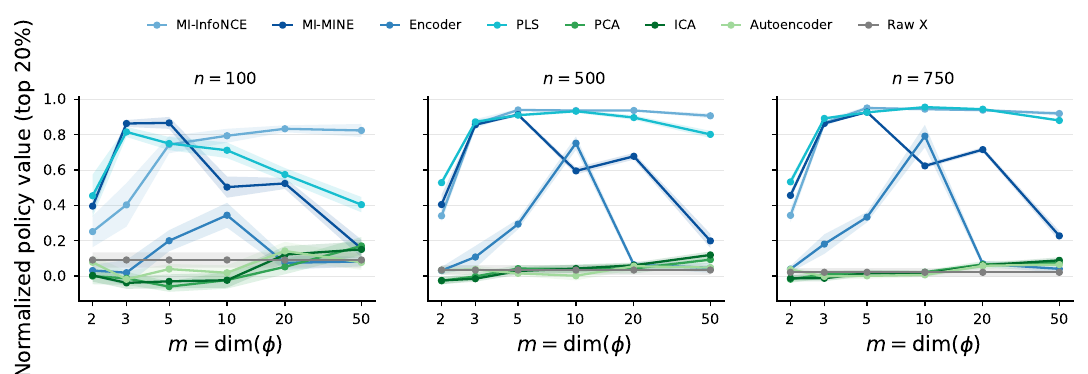}
    \caption{Normalized policy value (top 20\%), $\hat Y$ target.}
  \end{subfigure}\\[2pt]
  \begin{subfigure}[t]{\textwidth}\centering
    \includegraphics[width=\textwidth]{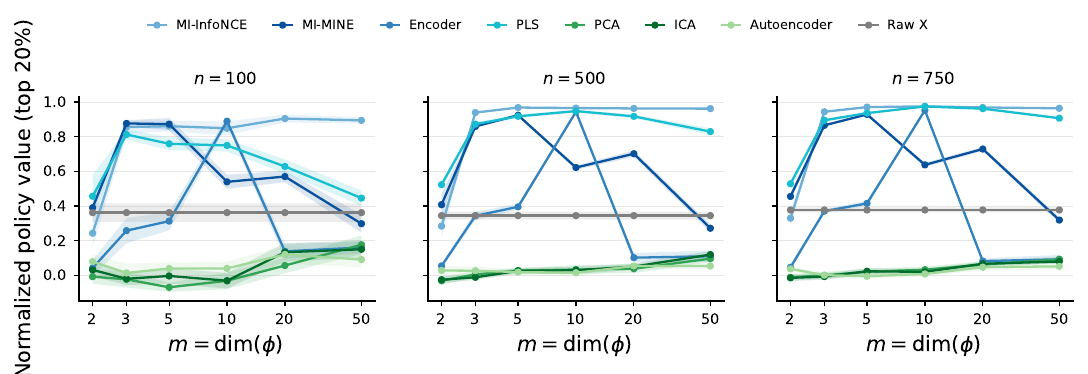}
    \caption{Normalized policy value (top 20\%), $Y$ target.}
  \end{subfigure}
  \caption{\textbf{Representation dimension sweep, DML --- normalized
    top-20\% policy value.} MI-InfoNCE and PLS recover most of the available ranking signal once $m$ is large enough to represent the outcome-relevant latent structure. Unsupervised methods remain near random across dimensions.}
  \label{fig:m-sweep-dml-policy}
\end{figure}

\paragraph{Results.} 
Figures~\ref{fig:m-sweep-dml-pehe} and~\ref{fig:m-sweep-dml-policy} report normalized PEHE and policy value across representation dimensions, respectively. The results show a familiar underfitting vs. estimation tradeoff. At $n=500$, increasing $m$ from 2 to 5 lowers MI-InfoNCE's normalized PEHE from $0.961$ to $0.538$ and raises its policy value from $0.341$ to $0.940$. Increasing $m$ beyond 5 provides no uniform gain: MI-InfoNCE remains consistently good at $m=10$ and $20$, whereas PLS begins to degrade at $m=50$ and the prediction encoder and MI-MINE are more sensitive to the choice of dimension. The main choice $m=10$ therefore lies within the effective range for the best-performing methods, but increasing representation dimensions has diminishing marginal returns, and for some methods, at small sample sizes, is even detrimental.

\subsection{Semi-synthetic (Medical) dataset}\label{app:semi_synthetic}
\subsubsection{Construction}
\paragraph{Data disclosure and integrity checks.}
The empirical records underlying this semi-synthetic dataset were anonymized before access and were obtained under IRB approval. We use them to obtain covariate, diagnosis, and outcome distributions. We synthetically generate treatment assignment and potential outcomes. 

\paragraph{Cohorts and covariates.}
The historical (EHR) cohort $H$ has $N_H=60{,}248$ samples. The experimental pool (mobile-app cohorts) $E$ contains $N_E=11{,}747$ samples. We retain covariates $X$ after excluding identifiers, variables that may lead to leakage of targets, features with excessive missingness, and all 37 diagnosis trajectories used to construct $S$. The resulting $X\in\mathbb{R}^{340}$ spans maternal demographics and insurance; obstetric history and pregnancy characteristics; prenatal measurements and health behaviors; diagnoses recorded across pregnancy; and emergency, inpatient, prenatal, and behavioral-health utilization. We use an $\ell_1$-logistic model fit on $H$ to identify 40 relevant coordinates with high signal used to parameterize the semi-synthetic outcome and treatment mechanisms.

\paragraph{Auxiliary outcomes.}
We define a pre-specified pool of 56 diagnosis and problem-list trajectories from the first, second, and third trimesters and the post-baseline/pre-postpartum window. These include anxiety, depression, hypertension, gestational hypertension, substance-use and behavioral diagnoses, bipolar disorder, obsessive-compulsive disorder, trauma reactions, diabetes, autoimmune conditions, cardiomyopathy, kidney and liver conditions, and gestational diabetes. To deal with missingness, we retain only 37 outcomes whose prevalence in $H$ lies between $0.5\%$ and $99.5\%$. They are standardized using $H$ only, after which we fit PCA to obtain the surrogates. The first five components define $S\in\mathbb{R}^{5}$ and explain $41.3\%$ of the historical outcome variance. These auxiliary outcomes are observed in both cohorts.

\paragraph{Synthetic treatment and potential outcomes.}
The baseline outcome score $\mu_X$ and auxiliary-outcome direction $\beta_S$ are fit using $H$ only. Let $S_i(0)$ denote the observed untreated auxiliary-outcome representation after a shared measurement-noise draw. Treatment shifts it along the outcome-relevant direction,
\[
\tau_S(X_i)=g(X_i)\frac{\beta_S}{\lVert\beta_S\rVert_2^2},
\qquad
S_i(t)=S_i(0)+t\,\tau_S(X_i),
\]
where $g(X)$ is a heterogeneous effect score depending on 10 of the 40 selected coordinates and normalized to have standard deviation $0.5$. Let $D_i(t)=1$ denote postpartum depression. Its potential-outcome probability satisfies
\[
\operatorname{logit}P\{D_i(t)=1\mid X_i,S_i(t)\}
=\mu_X(X_i)+\lambda_S S_i(t)^\top\beta_S,
\qquad \lambda_S=2.
\]
We report $Y_i(t)=1-D_i(t)$, so positive treatment effects correspond to a reduced probability of postpartum depression. The same equation is used for $H$ and $E$. Experimental treatment is $T\sim\operatorname{Bernoulli}(1/2)$, independently of $X$, and the observed pair is $(S_i(T_i),Y_i(T_i))$. Since $T$ affects $Y$ only through $S(T)$, surrogacy holds by construction. Using the corresponding noisy potential surrogate in each potential-outcome equation avoids a measurement-error-induced direct association between $T$ and $Y$ conditional on the observed $S$.

\paragraph{Methods and evaluation.}
We standardize $X$ using $H$, then train representations and 50-tree random-forest outcome prediction models on an i.i.d. historical sample of size $10{,}000$. Neural methods use the encoder $340\to64\to32\to\dim(\phi)$, Adam with learning rate $10^{-3}$ and batch size 256, and validation-based early stopping over at most 200 epochs. The main analysis uses $\dim(\phi)=10$; $\dim(\phi)=5$ is a sensitivity analysis. We evaluate the prediction encoder, conditional MINE, conditional InfoNCE, PLS, PCA, ICA, an autoencoder, and two raw-$X$ variants whose outcome models fit $h(X)$ or $h(X,S)$. For each seed, we reserve 5,000 of the 11,747 experimental-pool pregnancies for testing and draw every training sample from the disjoint remaining pool. We standardize the CATE features using each experimental training sample, remove constant coordinates, and clip standardized training and test values to $[-5,5]$. We then fit three-fold DML with 40-iteration histogram-gradient-boosted outcome and treatment nuisances and a ridge final stage. Experimental sample sizes are $n\in\{100,250,500,750,1000\}$. We repeat each experiment on 10 random data draws, using the same draws for every method, and report the mean and its standard error. We report normalized PEHE and normalized top-20\% policy value.

\subsubsection{Overlap and representation diagnostics}\label{app:semisynth_assumptions}

Unlike the fully synthetic experiment, the historical and experimental samples come from different empirical cohorts. Thus randomization, surrogacy, and the outcome-model comparability hold by construction, whereas overlap between the two cohorts and the sufficiency of the learned representation need not hold exactly. Hence we empirically assess overlap between $H$ and $E$ and whether raw $X$ retains predictive information not retained by $\phi(X)$. We summarize the empirical assessment in Table~\ref{tab:semisynth-diagnostics}.

\begin{table}[t]
\centering
\small
\begin{tabular}{lccc}
\toprule
\multicolumn{4}{l}{{\textbf{Cohort overlap}}} \\
\hline \hline
Cohort-classification AUC
    & \multicolumn{3}{c}{0.752} \\
NN coverage in $(X_{Y},S)$, $T=0$ / $T=1$
    & \multicolumn{3}{c}{0.940 / 0.951} \\
\midrule
\multicolumn{4}{l}{{\textbf{Representation sufficiency}}} \\
\hline \hline
Representation
    & Outcome $\Delta$ log loss
    & $S(0)$ $\Delta R^2$
    & $S(1)$ $\Delta R^2$ \\
\midrule
Prediction encoder     & 0.0065 & 0.0146 & 0.0379 \\
Conditional MINE       & 0.0089 & 0.0501 & 0.0772 \\
Conditional InfoNCE    & 0.0102 & 0.0568 & 0.0837 \\
PLS                    & 0.0018 & 0.0261 & 0.0443 \\
PCA                    & 0.0083 & 0.1656 & 0.1868 \\
ICA                    & 0.0102 & 0.1653 & 0.1853 \\
Autoencoder            & 0.0104 & 0.1747 & 0.1902 \\
\bottomrule
\end{tabular}
\caption{Diagnostics for cohort overlap and representation sufficiency. The final three columns report the cross-validated improvement from appending raw $X$ to the representation; smaller values indicate that less predictive information remains beyond the representation.}
\label{tab:semisynth-diagnostics}
\end{table}

We assess overlap between the historical and experimental samples using a cross-validated classifier trained to distinguish the two cohorts based on $X$. The classifier has an AUC of $0.752$, indicating some distribution shift. We then compare nearest-neighbor distances in $(X_{Y},S)$, where $X_{Y}\in \mathbb{R}^{40}$ are the covariates used to construct the semi-synthetic outcomes. We do not use the full covariates $X\in \mathbb{R}^{340}$ since Euclidean nearest-neighbor distance becomes increasingly difficult to interpret in higher dimensions. We split our historical samples into a reference set and a held-out set. For each held-out historical sample, we compute its distance to the nearest-neighbor in the reference set. The 95-th percentile of these distances defines our coverage threshold. We consider an experimental sample to be "covered" when its distance to historical reference set does not exceed this threshold. Coverage is $94.0\%$ among controls and $95.1\%$ among treated observations, indicating substantial but imperfect overlap in the variables used by the DGP.

To assess sufficiency (i), we test whether raw $X$ improves prediction of $Y^*$ after conditioning on $(\phi(X),S)$. We compare three-fold cross-validated predictions using $(\phi(X),S)$ and $(\phi(X),S,X)$ and report the resulting reduction in log loss. To assess sufficiency (ii), we similarly compare predictions of $S(0)$ and $S(1)$ using $\phi(X)$ and $(\phi(X),X)$ and report the increase in $R^2$. Both potential surrogates are available by construction in the semi-synthetic experiment. Under exact sufficiency, adding $X$ would provide no population-level improvement. Values near zero therefore indicate that little predictive information remains outside the representation.

For sufficiency (i), adding $X$ reduces outcome log loss by at most $0.0104$ across all representations. For sufficiency-(ii), the improvements for the supervised representations range from $0.0146$ to $0.0837$, compared with $0.1653$ to $0.1902$ for PCA, ICA, and the autoencoder. The prediction encoder retains the most information relevant to the potential surrogates, with PLS close behind it. These improvements are close to zero, particularly for the supervised representations, suggesting that raw $X$ contributes little predictive information beyond $\phi(X)$ and that sufficiency holds approximately in this experiment.

\subsubsection{Results and representation dimension}
Figure~\ref{fig:semisynth_phi10_full} and Tables~\ref{tab:corrected_semisynth_phi10_pehe}--\ref{tab:corrected_semisynth_phi10_policy} report the complete results at $\dim(\phi)=10$. Figure~\ref{fig:semisynth_phi5_full} reports the corresponding results at $\dim(\phi)=5$, while Table~\ref{tab:corrected_semisynth_phi_sensitivity} directly compares the prediction encoder at dimensions 5 and 10.
\begin{figure}[t]
\centering
\includegraphics[width=0.98\linewidth]{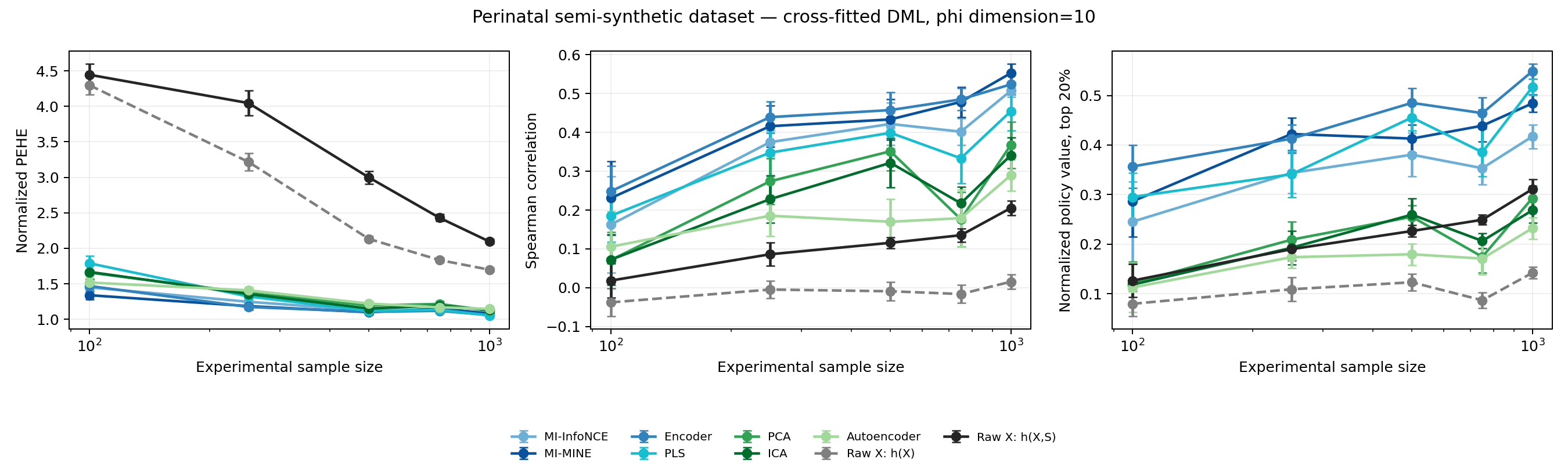}
\caption{The nine methods used in Section~\ref{sec:exp_semisynth} on the perinatal semi-synthetic dataset at $\dim(\phi)=10$. Points show means over 10 runs; error bars show $\pm 1$ standard error. Supervised representations learned from historical outcomes outperform both raw-$X$ variants and generally outperform unsupervised compression.}
\label{fig:semisynth_phi10_full}
\end{figure}

\begin{table}[t]
\centering
\resizebox{\linewidth}{!}{%
\begin{tabular}{lccccc}
\toprule
Method & $n=100$ & $n=250$ & $n=500$ & $n=750$ & $n=1000$ \\
\midrule
MI-InfoNCE & 1.451$\pm$0.076 & 1.247$\pm$0.047 & 1.124$\pm$0.031 & 1.133$\pm$0.020 & 1.090$\pm$0.019 \\
MI-MINE & \textbf{1.340$\pm$0.062} & 1.186$\pm$0.052 & \textbf{1.099$\pm$0.035} & 1.152$\pm$0.039 & 1.082$\pm$0.022 \\
Encoder & 1.470$\pm$0.068 & \textbf{1.173$\pm$0.034} & 1.106$\pm$0.026 & \textbf{1.119$\pm$0.028} & 1.059$\pm$0.017 \\
PLS & 1.790$\pm$0.102 & 1.322$\pm$0.062 & 1.118$\pm$0.030 & 1.135$\pm$0.031 & \textbf{1.055$\pm$0.015} \\
PCA & 1.659$\pm$0.093 & 1.372$\pm$0.061 & 1.193$\pm$0.049 & 1.216$\pm$0.036 & 1.123$\pm$0.022 \\
ICA & 1.668$\pm$0.102 & 1.356$\pm$0.044 & 1.151$\pm$0.046 & 1.192$\pm$0.023 & 1.130$\pm$0.017 \\
Autoencoder & 1.518$\pm$0.065 & 1.408$\pm$0.044 & 1.221$\pm$0.034 & 1.170$\pm$0.023 & 1.146$\pm$0.019 \\
Raw $X$, $h(X)$ & 4.299$\pm$0.137 & 3.217$\pm$0.122 & 2.133$\pm$0.035 & 1.835$\pm$0.033 & 1.699$\pm$0.026 \\
Raw $X$, $h(X,S)$ & 4.446$\pm$0.154 & 4.047$\pm$0.175 & 2.998$\pm$0.093 & 2.431$\pm$0.050 & 2.096$\pm$0.034 \\
\bottomrule
\end{tabular}}
\caption{Normalized PEHE on the perinatal semi-synthetic dataset at $\dim(\phi)=10$. Entries are means $\pm$ standard errors over 10 runs; lower is better.}
\label{tab:corrected_semisynth_phi10_pehe}
\end{table}

\begin{table}[t]
\centering
\resizebox{\linewidth}{!}{%
\begin{tabular}{lccccc}
\toprule
Method & $n=100$ & $n=250$ & $n=500$ & $n=750$ & $n=1000$ \\
\midrule
MI-InfoNCE & 0.245$\pm$0.081 & 0.344$\pm$0.040 & 0.381$\pm$0.043 & 0.353$\pm$0.033 & 0.417$\pm$0.024 \\
MI-MINE & 0.286$\pm$0.072 & \textbf{0.422$\pm$0.033} & 0.413$\pm$0.028 & 0.439$\pm$0.032 & 0.484$\pm$0.017 \\
Encoder & \textbf{0.357$\pm$0.044} & 0.413$\pm$0.028 & \textbf{0.485$\pm$0.029} & \textbf{0.464$\pm$0.031} & \textbf{0.549$\pm$0.015} \\
PLS & 0.295$\pm$0.048 & 0.341$\pm$0.047 & 0.456$\pm$0.027 & 0.386$\pm$0.049 & 0.517$\pm$0.016 \\
PCA & 0.121$\pm$0.043 & 0.209$\pm$0.036 & 0.255$\pm$0.023 & 0.175$\pm$0.032 & 0.293$\pm$0.025 \\
ICA & 0.118$\pm$0.042 & 0.193$\pm$0.034 & 0.260$\pm$0.033 & 0.207$\pm$0.015 & 0.268$\pm$0.025 \\
Autoencoder & 0.113$\pm$0.050 & 0.174$\pm$0.022 & 0.180$\pm$0.022 & 0.171$\pm$0.032 & 0.232$\pm$0.022 \\
Raw $X$, $h(X)$ & 0.080$\pm$0.026 & 0.109$\pm$0.024 & 0.123$\pm$0.017 & 0.087$\pm$0.016 & 0.142$\pm$0.012 \\
Raw $X$, $h(X,S)$ & 0.127$\pm$0.033 & 0.190$\pm$0.023 & 0.227$\pm$0.011 & 0.249$\pm$0.010 & 0.311$\pm$0.020 \\
\bottomrule
\end{tabular}}
\caption{Normalized top-20\% policy value on the perinatal semi-synthetic dataset at $\dim(\phi)=10$. Entries are means $\pm$ standard errors over 10 runs; higher is better.}
\label{tab:corrected_semisynth_phi10_policy}
\end{table}

For the prediction encoder, dimensions 5 and 10 give similar results. The ten-dimensional representation has lower normalized PEHE at four of the five sample sizes, although the differences are small, and neither dimension has uniformly higher policy value. We retain $\dim(\phi)=10$ for consistency with the synthetic experiment and report dimension 5 as a sensitivity analysis.

\begin{figure}[t]
\centering
\includegraphics[width=0.98\linewidth]{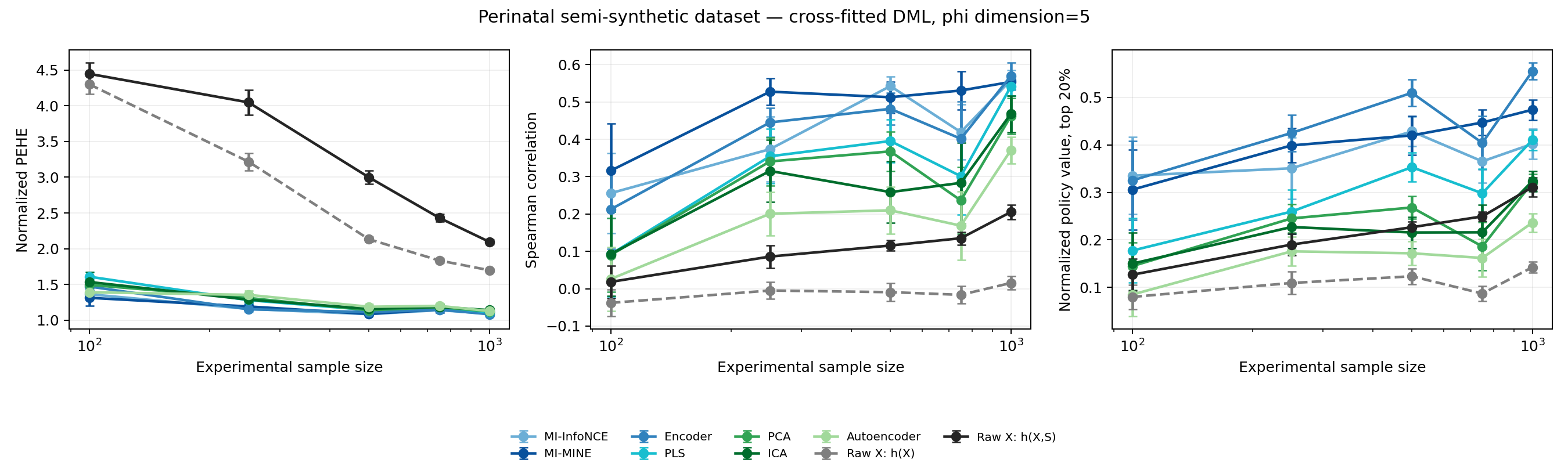}
\caption{The nine methods used in the main comparison at $\dim(\phi)=5$. The prediction encoder gives similar PEHE and policy value at dimensions 5 and 10.}
\label{fig:semisynth_phi5_full}
\end{figure}

\begin{table}[t]
\centering
\begin{tabular}{rcccc}
\toprule
& \multicolumn{2}{c}{Normalized PEHE} & \multicolumn{2}{c}{Policy@20\%} \\
$n$ & $\dim(\phi)=5$ & $\dim(\phi)=10$ & $\dim(\phi)=5$ & $\dim(\phi)=10$ \\
\midrule
100 & 1.471$\pm$0.105 & 1.470$\pm$0.068 & 0.325$\pm$0.083 & 0.357$\pm$0.044 \\
250 & 1.153$\pm$0.025 & 1.173$\pm$0.034 & 0.425$\pm$0.039 & 0.413$\pm$0.028 \\
500 & 1.114$\pm$0.029 & 1.106$\pm$0.026 & 0.510$\pm$0.028 & 0.485$\pm$0.029 \\
750 & 1.148$\pm$0.029 & 1.119$\pm$0.028 & 0.404$\pm$0.056 & 0.464$\pm$0.031 \\
1000 & 1.084$\pm$0.025 & 1.059$\pm$0.017 & 0.555$\pm$0.018 & 0.549$\pm$0.015 \\
\bottomrule
\end{tabular}
\caption{Representation-dimension sensitivity for the prediction encoder on the perinatal semi-synthetic dataset. Entries are means $\pm$ standard errors over 10 runs.}
\label{tab:corrected_semisynth_phi_sensitivity}
\end{table}

\subsubsection{True-outcome sensitivity}
Figure~\ref{fig:semisynth_phi10_true_y} and Tables~\ref{tab:corrected_semisynth_phi10_pehe_true_y}--\ref{tab:corrected_semisynth_phi10_policy_true_y} report results using the true experimental outcome $Y$. To separate representation quality from outcome-prediction error, we repeat the same out-of-sample DML evaluation using the true experimental $Y$ directly. At $n=100$, conditional MINE has the lowest normalized PEHE ($1.717\pm0.108$), compared with $6.833\pm0.323$ for raw $X$. At $n=1000$, PLS and the prediction encoder are nearly tied ($1.059\pm0.039$ and $1.063\pm0.023$), while raw $X$ remains at $3.414\pm0.062$. The advantage of low-dimensional supervised representations therefore persists when the true experimental outcome, rather than $\widehat Y$, is used as the CATE target.

\begin{figure}[t]
\centering
\includegraphics[width=0.98\linewidth]{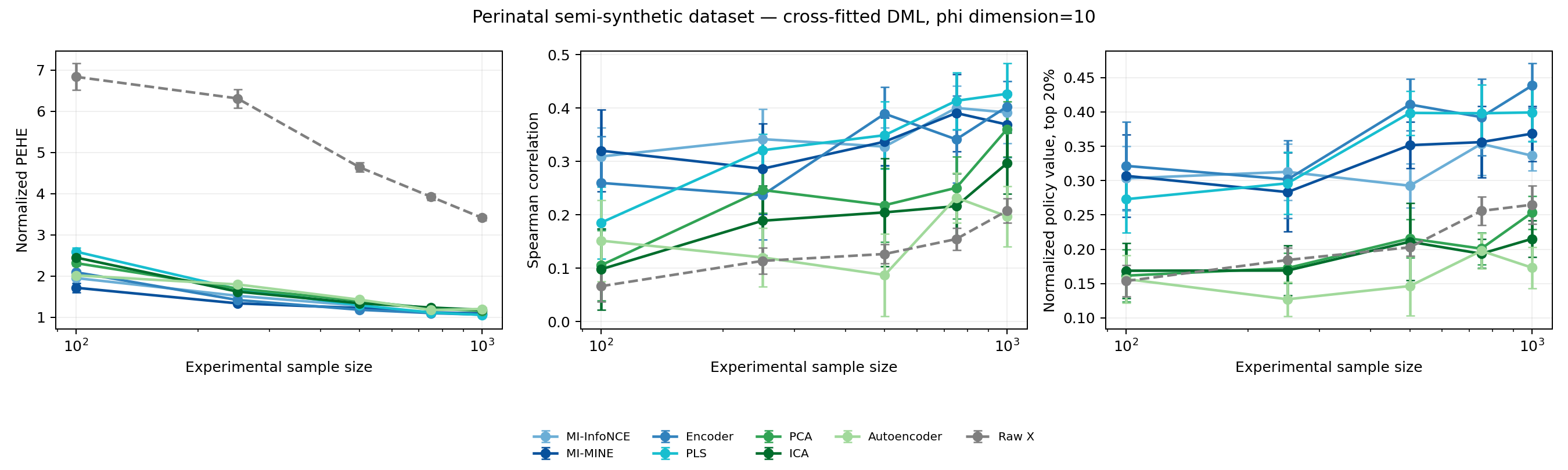}
\caption{The eight distinct methods at $\dim(\phi)=10$ using true experimental $Y$ as the CATE target. The $h(X,S)$ raw-$X$ variant is omitted because it is identical to raw $X$ when the true outcome is used directly. Points show means over 10 runs; error bars show $\pm 1$ standard error.}
\label{fig:semisynth_phi10_true_y}
\end{figure}

\begin{table}[t]
\centering
\resizebox{\linewidth}{!}{%
\begin{tabular}{lccccc}
\toprule
Method & $n=100$ & $n=250$ & $n=500$ & $n=750$ & $n=1000$ \\
\midrule
MI-InfoNCE & 1.952$\pm$0.117 & 1.522$\pm$0.077 & 1.282$\pm$0.041 & \textbf{1.096$\pm$0.026} & 1.104$\pm$0.043 \\
MI-MINE & \textbf{1.717$\pm$0.108} & \textbf{1.339$\pm$0.051} & 1.224$\pm$0.059 & 1.121$\pm$0.037 & 1.088$\pm$0.023 \\
Encoder & 2.094$\pm$0.153 & 1.420$\pm$0.078 & \textbf{1.183$\pm$0.032} & 1.102$\pm$0.023 & 1.063$\pm$0.023 \\
PLS & 2.594$\pm$0.088 & 1.686$\pm$0.140 & 1.292$\pm$0.058 & 1.116$\pm$0.023 & \textbf{1.059$\pm$0.039} \\
PCA & 2.315$\pm$0.172 & 1.698$\pm$0.093 & 1.400$\pm$0.031 & 1.214$\pm$0.027 & 1.157$\pm$0.035 \\
ICA & 2.451$\pm$0.168 & 1.624$\pm$0.096 & 1.340$\pm$0.039 & 1.237$\pm$0.031 & 1.181$\pm$0.037 \\
Autoencoder & 2.018$\pm$0.076 & 1.797$\pm$0.060 & 1.427$\pm$0.035 & 1.179$\pm$0.057 & 1.202$\pm$0.029 \\
Raw $X$ & 6.833$\pm$0.323 & 6.309$\pm$0.225 & 4.638$\pm$0.113 & 3.922$\pm$0.080 & 3.414$\pm$0.062 \\
\bottomrule
\end{tabular}}
\caption{Normalized PEHE on the perinatal semi-synthetic dataset at $\dim(\phi)=10$ using true experimental $Y$. Entries are means $\pm$ standard errors over 10 runs; lower is better.}
\label{tab:corrected_semisynth_phi10_pehe_true_y}
\end{table}

\begin{table}[t]
\centering
\resizebox{\linewidth}{!}{%
\begin{tabular}{lccccc}
\toprule
Method & $n=100$ & $n=250$ & $n=500$ & $n=750$ & $n=1000$ \\
\midrule
MI-InfoNCE & 0.303$\pm$0.047 & \textbf{0.313$\pm$0.041} & 0.293$\pm$0.032 & 0.354$\pm$0.046 & 0.336$\pm$0.022 \\
MI-MINE & 0.307$\pm$0.060 & 0.283$\pm$0.057 & 0.352$\pm$0.034 & 0.356$\pm$0.052 & 0.368$\pm$0.040 \\
Encoder & \textbf{0.321$\pm$0.064} & 0.302$\pm$0.056 & \textbf{0.411$\pm$0.038} & 0.392$\pm$0.056 & \textbf{0.438$\pm$0.032} \\
PLS & 0.273$\pm$0.049 & 0.296$\pm$0.045 & 0.399$\pm$0.032 & \textbf{0.398$\pm$0.042} & 0.399$\pm$0.043 \\
PCA & 0.162$\pm$0.038 & 0.173$\pm$0.022 & 0.216$\pm$0.028 & 0.201$\pm$0.023 & 0.253$\pm$0.024 \\
ICA & 0.169$\pm$0.040 & 0.169$\pm$0.036 & 0.211$\pm$0.056 & 0.193$\pm$0.021 & 0.215$\pm$0.026 \\
Autoencoder & 0.157$\pm$0.035 & 0.127$\pm$0.025 & 0.146$\pm$0.043 & 0.198$\pm$0.026 & 0.173$\pm$0.030 \\
Raw $X$ & 0.154$\pm$0.023 & 0.184$\pm$0.018 & 0.203$\pm$0.013 & 0.256$\pm$0.021 & 0.265$\pm$0.028 \\
\bottomrule
\end{tabular}}
\caption{Normalized top-20\% policy value on the perinatal semi-synthetic dataset at $\dim(\phi)=10$ using true experimental $Y$. Entries are means $\pm$ standard errors over 10 runs; higher is better.}
\label{tab:corrected_semisynth_phi10_policy_true_y}
\end{table}

\clearpage

\end{document}